\documentclass[letterpaper,twocolumn,10pt]{article}

\usepackage{usenix2019_v3}
\usepackage{tikz}
\usepackage{amsmath}
\usepackage{filecontents}
\usepackage{amsthm}
\usepackage{autobreak}
\allowdisplaybreaks
\usepackage{adjustbox}
\usepackage{booktabs}
\usepackage{makecell}
\usepackage{siunitx}
\usepackage{multirow}
\usepackage{balance}
\usepackage{longtable}
\usepackage{algorithm}
\usepackage[algo2e]{algorithm2e}
\usepackage{algorithmic}
\usepackage{bm}
\usepackage{graphicx}
\usepackage{placeins}
\usepackage{amsmath}
\DeclareMathOperator*{\argmax}{arg\,max}

\usepackage{booktabs}
\usepackage{tabularx}
\usepackage{listings}
\usepackage{enumitem}
\usepackage{hyperref}
\usepackage{wasysym}
\usepackage{amssymb}
\usepackage{array}
\usepackage{tabularx}
\usepackage{caption}
\usepackage{xcolor}
\usepackage{bm}
\usepackage{amsthm}
\newtheorem{theorem}{Theorem}
\newtheorem{lemma}{Lemma}[section]
\newtheorem{definition}[lemma]{Definition}

\hypersetup{
  colorlinks,
  linkcolor={blue!70!green},
  citecolor={green!70!blue},
  urlcolor={orange!70!red}
}

\newcommand{\para}[1]{\vspace{2pt}\noindent\textbf{{#1}}}
\newcommand{\ignore}[1]{}

\def\ourmethod{DPAdapter}

\begin{document}
\date{}

\title{\Large \bf DPAdapter: Improving Differentially Private Deep Learning through \\ Noise Tolerance Pre-training}

\author{
    {\rm Zihao Wang\textsuperscript{1}\thanks{The first two authors contributed equally to this work.}~,
    Rui Zhu\textsuperscript{1}\footnotemark[1]~,
    Dongruo Zhou\textsuperscript{1},
    Zhikun Zhang\textsuperscript{3},
    John Mitchell\textsuperscript{2},
    } \\
    {\rm Haixu Tang\textsuperscript{1},
    and XiaoFeng Wang\textsuperscript{1}
    }
    \\
    \textsuperscript{1}Indiana University Bloomington
    ~\textsuperscript{2}Stanford University \\
    \textsuperscript{3}CISPA Helmholtz Center for Information Security
    % \\
    % \{zwa2, zhu11, dz13\}@iu.edu, \{zhikun, John.Mitchell\}@stanford.edu, \{hatang, xw7\}@indiana.edu
}

\maketitle

\begin{abstract}
Recent developments have underscored the critical role of \textit{differential privacy} (DP) in safeguarding individual data for training machine learning models. 
However, integrating DP oftentimes incurs significant model performance degradation due to the perturbation introduced into the training process, presenting a formidable challenge in the {differentially private machine learning} (DPML) field.
To this end, several mitigative efforts have been proposed, typically revolving around formulating new DPML algorithms or relaxing DP definitions to harmonize with distinct contexts. 
In spite of these initiatives, the diminishment induced by DP on models, particularly large-scale models, remains substantial and thus, necessitates an innovative solution that adeptly circumnavigates the consequential impairment of model utility.

In response, we introduce \ourmethod{}, a pioneering technique designed to amplify the model performance of DPML algorithms by enhancing parameter robustness.
The fundamental intuition behind this strategy is that models with robust parameters are inherently more resistant to the noise introduced by DP, thereby retaining better performance despite the perturbations. 
\ourmethod{} modifies and enhances the sharpness-aware minimization (SAM) technique, utilizing a two-batch strategy to provide a more accurate perturbation estimate and an efficient gradient descent, thereby improving parameter robustness against noise. 
Notably, \ourmethod{} can act as a plug-and-play component and be combined with existing DPML algorithms to further improve their performance.
Our experiments show that \ourmethod{} vastly enhances state-of-the-art DPML algorithms, increasing average accuracy from 72.92\% to 77.09\% with a privacy budget of $\epsilon=4$.
\end{abstract}

\section{Introduction}
Recent years have witnessed an exponential growth in applications of \textit{deep neural networks} (DNNs) to various domains~\cite{ResNet, RCNN, NLP}.
However, DNN models trained with standard pipelines can be attacked by an adversary that seeks to reveal the data on which the model was trained. 
For example, Carlini et al.~\cite{Extracting} show that adversaries can generate and detect text sequences from the training set of a large transformer language model, while Balle et al.~\cite{Reconstructing} demonstrate that powerful adversaries can reconstruct images from the training set of a classifier. 
Alongside other results~\cite{ML-Doctor, First_Principles, Label-Only}, these studies indicate that models trained on sensitive datasets present a significant privacy risk.
\textit{Differential privacy} (DP)~\cite{DP} has been the golden standard for effective control of the risks of exposing training examples, and has already been adopted by various machine learning tasks~\cite{DP_finetune, DP_Transfer, LLM_DP, Unlocking} for protection. 
A differentially private algorithm is a randomized algorithm providing a formal guarantee that any single example in the training set only has a limited impact on the output distributions of the algorithm.
The privacy guarantee, denoted ($\epsilon$, $\delta$)-DP, is defined by two parameters ($\epsilon$, $\delta$), which we refer to as the \textit{privacy budget}. 
The smaller these two parameters are, the closer the output distributions between the training sets that differ by a single example, and therefore the more difficult it becomes for an adversary to infer whether any single data point is included during training.

\para{Challenges in Private Deep Learning.} 
\textit{Differentially private stochastic gradient descent} (DP-SGD)~\cite{DPSGD} stands out as a prevalent DPML technique. DP-SGD modifies the conventional mini-batch gradient calculation used in SGD by incorporating a privatized version, whereby the gradient of each training sample is clipped to a maximum norm. 
Subsequently, Gaussian noise, proportional to the clipping norm, is introduced to the sum of the clipped gradients,  masking the influence of any individual example on the sum. 
However, the incorporation of the DP noise often comes with a notable degradation in the performance of trained models~\cite{DPSGD, more_data}. 
In response, efforts to mitigate the adverse impact of DP on model utility have been made, either through formulating new algorithms~\cite{Semi-supervised_DP, more_data, Overbill, Private-kNN} or by relaxing the definition of DP to align with particular contexts~\cite{Theory_meets_Practice, New_Settings, label_DP}. 
Nonetheless, the detriment imposed by DP noise on differentially private models remains significant.

\para{Our Solution.} 
To mitigate the adverse impact of DP on model utility, prior research has underscored that transfer learning~\cite{transfer} from public data markedly enhances the model performance of DPML algorithms~\cite{Pretrain_DP}. 
In our research, we went a step further. 
Specifically, we observed that the deep learning training process is highly sensitive to noise, with even a small amount of noise leading to a substantial impact on performance. 
However, this issue can be mitigated by the model pre-trained to minimize parameter sensitivity to noise, which we call \textit{parameter robustness}, so as to control DP's performance impact. 
Therefore, we designed a new technique that pre-trains a model with its parameters robust to noise to enhance the performance of the downstream model fine-turned with DPML on private data. 

A technique that could serve this purpose is \textit{sharpness-aware minimization} (SAM), which augments parameter robustness~\cite{sam,AWP,AMP} for the purpose of enhancing a model's generality.  
This technique, however, turns out to be less effective in controlling the performance impact of DPML, since it is not designed for maximizing parameter robustness. 
Concretely, SAM adds the worst-case perturbation to parameters before computing gradients and later removes the perturbation after each round of parameter updates.  
For this purpose, it calculates both the perturbation and gradients on a small batch of training samples. 
A problem is that although a small batch could be enough for improving a model's generality, it is inadequate for making an accurate estimate of the worst-case perturbation, thereby rendering the model parameters less robust than they could be. 
To address this issue, we enhanced SAM for our purpose with a new technique called \ourmethod{}, which utilizes two batches of training instances, a large one for a more accurate estimate of the perturbation and a small one for effective gradient descent to ensure convergence.  
We further theoretically analyzed why this approach improves parameter robustness against noise, thereby reducing DPML's performance impact on the downstream model fine-tuning.   

\para{Empiricaly Results.} 
We implemented \ourmethod{} and evaluated its performance using a model pre-trained on CIFAR-100, which was then fine-tuned for classification tasks on CIFAR-10, SVHN, and STL-10 datasets. 
In these experiments, we utilized \ourmethod{} as the pre-training method and the CIFAR-100 dataset as the public dataset for the pre-training. 
The downstream tasks involve fine-tuning the pre-trained model with DPML. 
Our experimental results consistently show a stable improvement in the accuracy achieved by the downstream tasks, compared with those fine-tuned from the models without robustness enhancement. 
For instance, when the privacy budget is set to $\epsilon=4$ and DPML configured to DP-SGD (a popular setting), \ourmethod{} elevates the average accuracy across three downstream tasks to 77.09\%. 
By comparison, when utilizing the pre-trained models without \ourmethod{}, the accuracy is 72.92\%, over 4\% below that of our approach. 
Note that all existing DPML enhancements could only achieve a performance gain no more than 3\%~\cite{DPMLBench} when the privacy budget is set to $\epsilon=4$. 
Also, many of them cannot work together, since they all focus on the downstream fine-tuning step. Our approach, however, is designed for more generic protection, independent of tasks, and compatible with these existing solutions, given its focus on the pre-training step.    

\para{Contributions.}
Our key contributions are outlined below:

\vspace{2pt}\noindent$\bullet$\textit{~New technique}.
We developed \ourmethod{}, a new technique for enhancing parameter robustness, which leads to a general solution significantly outperforming yet compatible with existing techniques for controlling DPML's negative impacts. 

\vspace{2pt}\noindent$\bullet$\textit{~Theoretical understandings}.
We theoretically analyzed our solution to justify its effectiveness, unveiling intrinsic relations among parameter robustness, transferability, and DPML's performance impacts. Our analysis leads to new insights about how a pre-trained model can be designed to maximize the benefits of DPML.

\vspace{2pt}\noindent$\bullet$\textit{~Extensive empirical studies}.
We conducted comprehensive empirical studies of \ourmethod{} on various downstream tasks and DPML algorithms under different privacy budgets.
Our results show that our approach effectively enhances the performance of the downstream tasks fine-tuned by DPML.

\para{Roadmap}. 
The rest of the paper is organized as follows. \autoref{sec:background} presents the background of our research; \autoref{sec:Problem Formulation and Key Observations} presents our key observations; \autoref{sec:method} elaborates our method and theoretic analysis; \autoref{sec:eval} reports the evaluation of our techniques; \autoref{sec:discussion} discusses the limitations of our study and potential future work; \autoref{sec:related} reviews the related prior studies and \autoref{sec:conclusion} concludes the paper.
\vspace{-2pt}
\section{Background}
\label{sec:background}
\vspace{-1pt}
\subsection{Transfer Learning}
In the deep learning landscape, a particularly effective approach to model training involves a two-step process: \textit{pre-training} and \textit{fine-tuning}. 
This methodology, central to \textit{transfer learning}~\cite{transfer, transfer_theory}, allows practitioners to leverage the power of large-scale data and subsequently specialize a model for domain-specific tasks.

Pre-training serves as the first phase of transfer learning, often leveraging unsupervised or self-supervised learning to train a model on a large, general dataset~\cite{ViT, iGPT, Momentum, MoCoV2, MoCoV3}. The objective is to extract useful, high-level features from the data without necessarily focusing on a specific task.
The pre-trained model acts as a feature extractor and captures generalized knowledge from the initial dataset. This generalization facilitates transfer learning, where the knowledge gained from one task can be transferred to a different but related task.

The second phase, fine-tuning, builds upon the pre-trained model to specialize it for a specific task. In this phase, the model is trained on a smaller, labeled dataset that is directly related to the target task. The key innovation here is the transfer of knowledge from the pre-training phase; instead of initializing the neural network with random weights, the model starts with the weights obtained during pre-training. Because the model has already learned relevant features, it typically requires less data and fewer epochs to fine-tune effectively. Fine-tuning essentially adapts the generalized learning from the pre-training phase to a specific task, demonstrating the power of transfer learning to yield high performance with comparatively less data and computational effort.

\vspace{-1pt}
\subsection{Differentially Private Machine Learning}

\textit{Differential privacy}~\cite{DP} (DP) serves as a mathematical structure intended to quantify the assurance of privacy within data analysis, providing insights into the preservation of privacy when an individual’s data becomes part of a publicly analyzed dataset. 
This concept typically involves a privacy budget, represented by $\epsilon$, where a larger value signifies augmented privacy protection, and an optional failure probability, $\delta$, defines the permissible deviation from impeccable privacy.
Formally, we can define differential privacy as:

\begin{definition}[($\epsilon$, $\delta$)-DP] 
Given two neighboring datasets $D$ and $D'$ differing by one record, a mechanism $\mathcal{M}$ satisfies ($\epsilon$, $\delta$)-differential privacy if
$$ 
Pr[\mathcal{M}(D)\in S] \leq e^{\epsilon } \cdot Pr[\mathcal{M}(D') \in S] + \delta,
$$
where $\epsilon$ is the privacy budget, and $\delta$ is the failure probability.
\end{definition}

\para{Gaussian Mechanism.}
There are several approaches for designing mechanisms that satisfy $(\epsilon, \delta)$-differential privacy, among which the Gaussian mechanism is the most widely used one.
It computes a function $f$ on the dataset $D$ in a differentially private manner, by adding to $f(D)$ a random noise. 
The magnitude of the noise depends on $\Delta_f$, the \emph{global sensitivity} or the $\ell_2$ sensitivity of $f$.  
Such a mechanism $\mathcal{M}$ is given below:
$$
\begin{array}{crl}
& \mathcal{M}(D) & =f(D)+\mathcal{N}\left(0, \Delta_f^2 \sigma^2 \mathbf{I} \right)
\end{array}
$$
where $\Delta_f  = \max\limits_{(D,D') : D \simeq D'} || f(D) - f(D')||_2$, $\mathcal{N} (0, \Delta_f^2 \sigma^2 \mathbf{I})$ denotes a multi-dimensional random variable sampled from the normal distribution with mean $0$ and standard deviation  $\Delta_f \sigma$, and $\sigma=\sqrt{2\ln\frac{1.25}{\delta}}/\epsilon$.

\para{DP-SGD.}
The integration of differential privacy into deep learning aims to build models that can learn from data without compromising the privacy of individuals within the dataset.
Differentially private stochastic gradient descent (DP-SGD)~\cite{DPSGD} is the most widely used algorithm to enforce DP guarantee for the deep learning models. 
It adapts the standard SGD algorithm by introducing a few privacy-preserving modifications: \textit{gradient clipping} and \textit{noise addition}.

The gradient clipping operation aims to limit the sensitivity of each gradient, ensuring that a single data point does not unduly influence the model's learning process.
Following gradient clipping, Gaussian noise is added to the clipped gradients before they are used to update the model parameters. 
This noise ensures that the exact values of the gradients—which could reveal sensitive information—are masked.

\para{Privacy Budget Composition.}
On a broader spectrum, the overall privacy budget calculation of DP-SGD is established by demonstrating the privacy budget of each iteration under certain $(\epsilon, \delta)$ values, followed by applying amplification by subsampling and composition throughout the iterations.

\vspace{-1pt}
\subsection{Adversarial Robustness and Model Parameter Robustness}
\label{subsec:Adversarial Robustness and Model Parameter Robustness}

\para{Adversarial Training (AT).} 
It is widely regarded as the most effective defense method against adversarial attacks~\cite{AE_1, Adversarial, Intriguing, Limitations_AE, Evasion}.
AT essentially employs a ``fight fire with fire'' approach. It introduces adversarial examples into the training set, which are slight, carefully calculated modifications of the original inputs. These modifications are designed to mislead the model into making erroneous predictions. By learning from these adversarial examples, the model improves its ability to correctly classify such inputs in the future, thereby enhancing its overall resilience and robustness against adversarial attacks.

\para{Parameter Robustness.} 
It typically refers to the sensitivity of a model to small perturbations in its parameters. 
When discussing parameter robustness, we are interested in the model output (e.g., predictions on testing data) changes when its parameters are subjected to some level of noise. 

Specifically, given a model \( f \) parameterized by \( \theta \) and a test sample \( x \), the robustness of the model with respect to its parameters can be defined as:

\begin{equation*}
\label{eq:para robust}
    \rho(f) = \max_{\boldsymbol{x}, \theta, \Delta}|f_{\theta+\Delta}(\boldsymbol{x}) - f_{\theta}(\boldsymbol{x})|/\|\Delta\|_2
\end{equation*}

\noindent where $\rho(f)$ represents the maximum change in the model's output.
 \( \Delta \) is the noise added to the model parameters \( \theta \).
 \( \| \cdot \| \) denotes a norm, for instance, the L2 norm.

This definition provides a measure of the potential variation in the model's output when its parameters undergo perturbation due to noise. 
A diminutive value of \( \rho \) is indicative of pronounced parameter robustness. Subsequently, we aim to \textit{enhance} the parameter robustness of \(f\), denoted by the objective of reducing the value of \(\rho(f)\).

\para{Robust Accuracy.}
We resort to \textit{robust accuracy} as an evaluation metric for parameter robustness in the classification paradigm; Given a specific model \( f_{\theta} \), we introduce Gaussian noise, represented as \( N(0, 0.1) \), uniformly across parameters spanning all layers. The resultant perturbed model is represented as \( f_{\hat{\theta}} \). 
The model's performance, when faced with this perturbation, is then assessed. Consequently, the robust accuracy is computed over a dataset \( D \) and is delineated as:

\begin{equation*}
    \text{Robust Accuracy} = \frac{1}{|D|}\sum_{x \in D} I(f_{\hat{\theta}}(x) = y)
\end{equation*}

\noindent where \( I \) is the indicator function that returns 1 when the condition inside is true and 0 otherwise.
A model exhibiting a higher robust accuracy indicates stronger parameter robustness, corresponding to a lower value of \(\rho\).

The optimization of parameter robustness can be achieved through a technique known as \textit{sharpness-aware minimization} (SAM)~\cite{sam,AWP,AMP}. This method typically involves computing the worst-case perturbation on the parameters. The optimization then takes place on the model parameters subjected to this worst-case perturbation. The intention is to ensure that the model retains a significant portion of its performance, even when subjected to the worst-case perturbation. However, this technique is not designed to maximize parameter robustness but to achieve better generality. Consequently, it proves to be less effective in mitigating the performance impact of DPML (\autoref{subsec:results}).

\section{Problem Formulation and Key Observations}
\label{sec:Problem Formulation and Key Observations}

\subsection{Problem Formulation}
\label{subsec:Problem Formulation}

\para{Threat Model.} 
We consider a scenario where an attacker applies privacy attacks~\cite{Extracting, Reconstructing, ML-Doctor, First_Principles, Label-Only} to unveil the training data of a model. 
To defend against the attacks, the defender endeavors to build the model using differential privacy, which provably protects sensitive data while preserving its original utility. Additionally, the defender possesses a public dataset related to the private data, aligning with the standard assumption of transfer learning that the source and the target datasets follow similar but not identical distributions.

\para{Problem Definition.} 
Our aim is to build a pre-trained model that can be used for improving both the accuracy and privacy protection of the models trained for downstream tasks. Specifically, to build a pre-trained model that serves as a beneficial initialization for DP fine-tuning, we seek to answer two key questions: 1) what properties the pre-trained model should exhibit for the benefit of downstream models? and consequently 2) how the pre-training of the model should be executed to acquire these properties?

\subsection{Motivation and Key Observations}
\label{subsec:Key Observations}

\para{Motivation.}
Our motivation to enhance the performance of the DPML algorithm using pre-trained models is rooted in our understanding of adversarial robustness training. The traditional adversarial training (AT) incrementally conditions the model to resist perturbations of its input. However, in the DPML algorithm, the noise is predominantly introduced into the gradients, or equivalently, into the model's parameters.
Hence, we ask if there is any technique similar to AT that may condition the model to resist perturbations of its parameters. Obviously, if this can be achieved, the impact of the noise introduced by the DPML algorithm on the model performance could be minimized. It turns out, such an AT method has been proposed previously, referred to as the {\em sharpness aware
minimization (SAM)}, which confers noise resilience to parameters \autoref{subsec:Adversarial Robustness and Model Parameter Robustness}.
Based on the premise, we introduce two working hypotheses.

We first consider a specialized transfer learning scenario where the training datasets for the upstream and downstream tasks are i.i.d. (independent and identically distributed); for instance, each of these training datasets represent a random sub-sample of an overarching dataset. Under this scenario, we propose the first hypothesis:

\para{Hypothesis 1.}
\textit{Assuming the datasets $A$ and $B$ are randomly sampled from the same distribution, if a pre-trained model built from the dataset $A$ possesses enhanced parameter robustness, the model fine-tuned on the dataset $B$ using the DPML algorithm under a fixed privacy budget will yield better performance.}

Next, we extend our hypothesis into a more general transfer learning scenario, where the upstream and downstream datasets are not identical but related, for example, when the upstream training data comes from ImageNet and the downstream data is sourced from CIFAR-10. Under this scenario, we have our second hypothesis.

\para{Hypothesis 2.}
\textit{Assuming the datasets $A$ and $B$ are sampled from two different but similar distributions, respectively, if a pre-trained model built from the dataset $A$ exhibits strong parameter robustness, the model fine-tuned on the dataset $B$ using the DPML algorithm under a fixed privacy budget will yield better performance.}

\para{Rationales.}
Hypothesis 1 is quite intuitive. Drawing parallels from AT — which can counteract the effects of adversarial examples by introducing noise into the input — we anticipate that SAM can similarly mitigate the implications of the noise added to parameters by the DPML algorithm. 

Hypothesis 2, on the other hand, is inspired by prior work~\cite{qu2023reaas}. It has been observed that pre-trained models, when subjected to AT, can impart some degree of their acquired adversarial robustness to downstream models via transfer learning. This led us to a deeper contemplation: Can the parameter robustness garnered through SAM training also be transferred to downstream tasks through transfer learning? If this transfer of parameter robustness is indeed feasible, then, based on Hypothesis 1, it could potentially alleviate the performance degradation in downstream training using the DPML algorithm.

\para{Empirical Validation.}
To evaluate our hypotheses, we utilized state-of-the-art Adversarial Model Perturbation (AMP) methodologies. Our models were exposed to an array of AMP intensities, producing pre-trained models that span a spectrum of parameter robustness. We then measured their parameter robustness by assessing the model accuracy on testing data after perturbation, and analyzed their performance in conjunction with DP-SGD.

\begin{figure}[!tbp]
\centerline{\includegraphics[width=\linewidth]{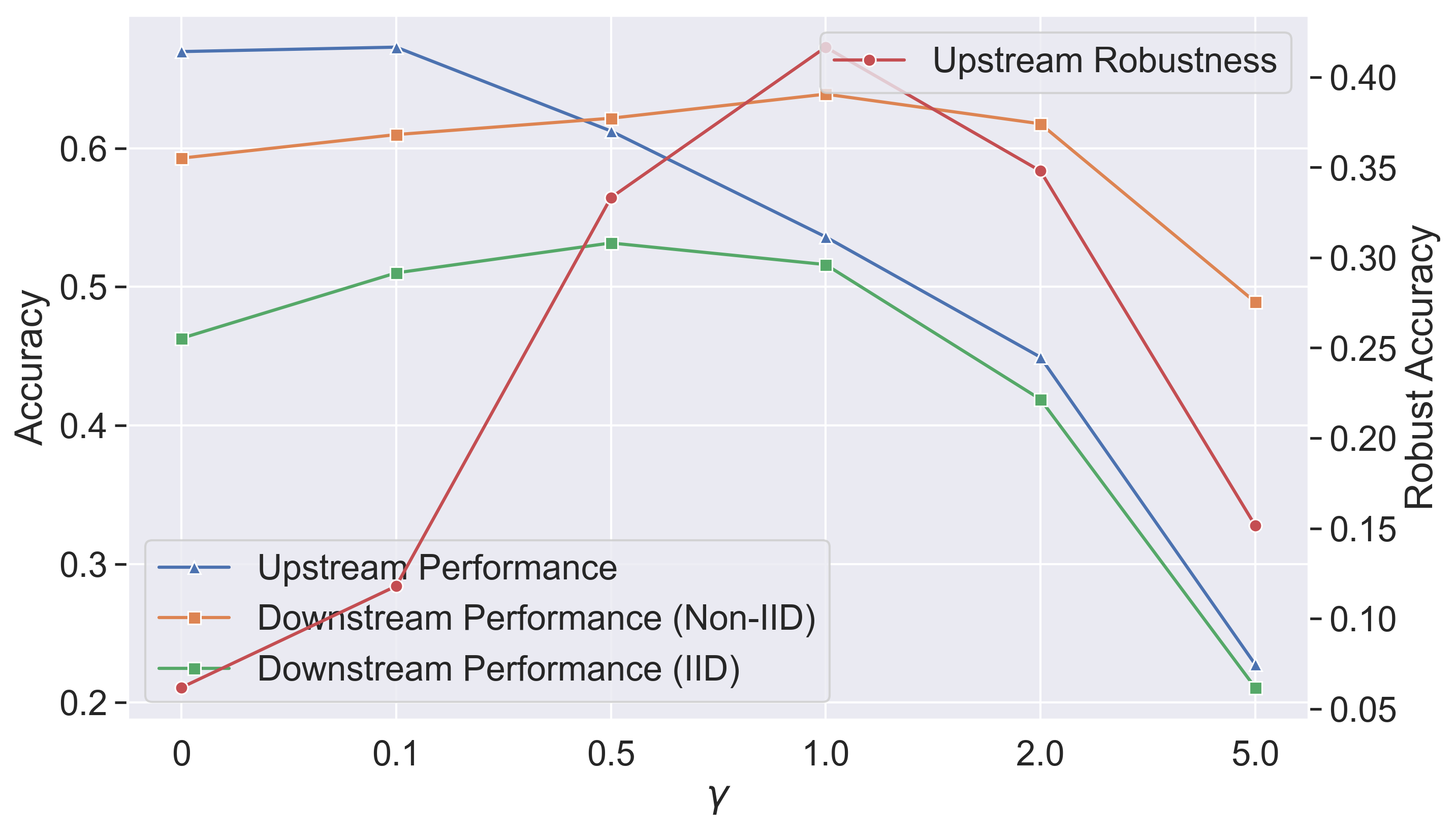}}
\caption{Impact of perturbation magnitude on AMP. }
\label{fig:gamma_amp}
\end{figure}

Specifically, to validate Hypothesis 1, we examined a scenario where both upstream and downstream training data were sourced from the CIFAR-100 dataset, with the upstream data constituting 90\% of the CIFAR-100 training dataset and the downstream data constituting the remaining 10\%. As depicted in \autoref{fig:gamma_amp}, the x-axis signifies the parameter $\gamma$ in AMP, which modulates the intensity of parameter robustness. By adjusting $\gamma$, we derived pre-trained models of various levels of robustness (indicated by the red line). This robustness peaks with increasing $\gamma$ and then starts to wane. Correspondingly, the green line, which illustrates the accuracy of the model's performance after downstream optimization using DP-SGD, displays an initial rise followed by a decline, mirroring the trajectory of parameter robustness. This result suggests that if the upstream and downstream training datasets both follow an i.i.d. distribution, the parameter robustness of upstream model is directly correlated with the performance of the downstream fine-tuned model.

For Hypothesis 2, we used the upstream training data comprising the same 90\% of the CIFAR-100 dataset, while the downstream training data was entirely from the CIFAR-10 dataset. Still referencing \autoref{fig:gamma_amp}, the blue line indicates the performance (accuracy) of the downstream model after DP-SGD optimization. Consistent with the results above, the blue line showcases an initial increase in performance of the downstream model followed by a significant drop, which is highly consistent with the pattern of the parameter robustness. This result suggests that in a transfer learning scenario, the parameters robustness of an upstream pre-trained model can influence the performance of the downstream model trained using the DPML algorithm. Additionally, we observed a gradual decline in the accuracy of the upstream model with an increase in its parameter robustness, indicating a trade-off between the parameter robustness and the accuracy of the model. If the accuracy of the upstream model drops too significantly, it can detrimentally impact the performance of the downstream model. Taking these observation into account, we outline our design objectives below, aiming to illustrate our intent of devising a cost-effective strategy to build a pre-trained model that strikes a balance between two goals, i.e., to increase  the parameter robustness while maintaining the high accuracy of the pre-trained model.

\begin{figure}[!tbp]
\centerline{\includegraphics[width=.75\linewidth]{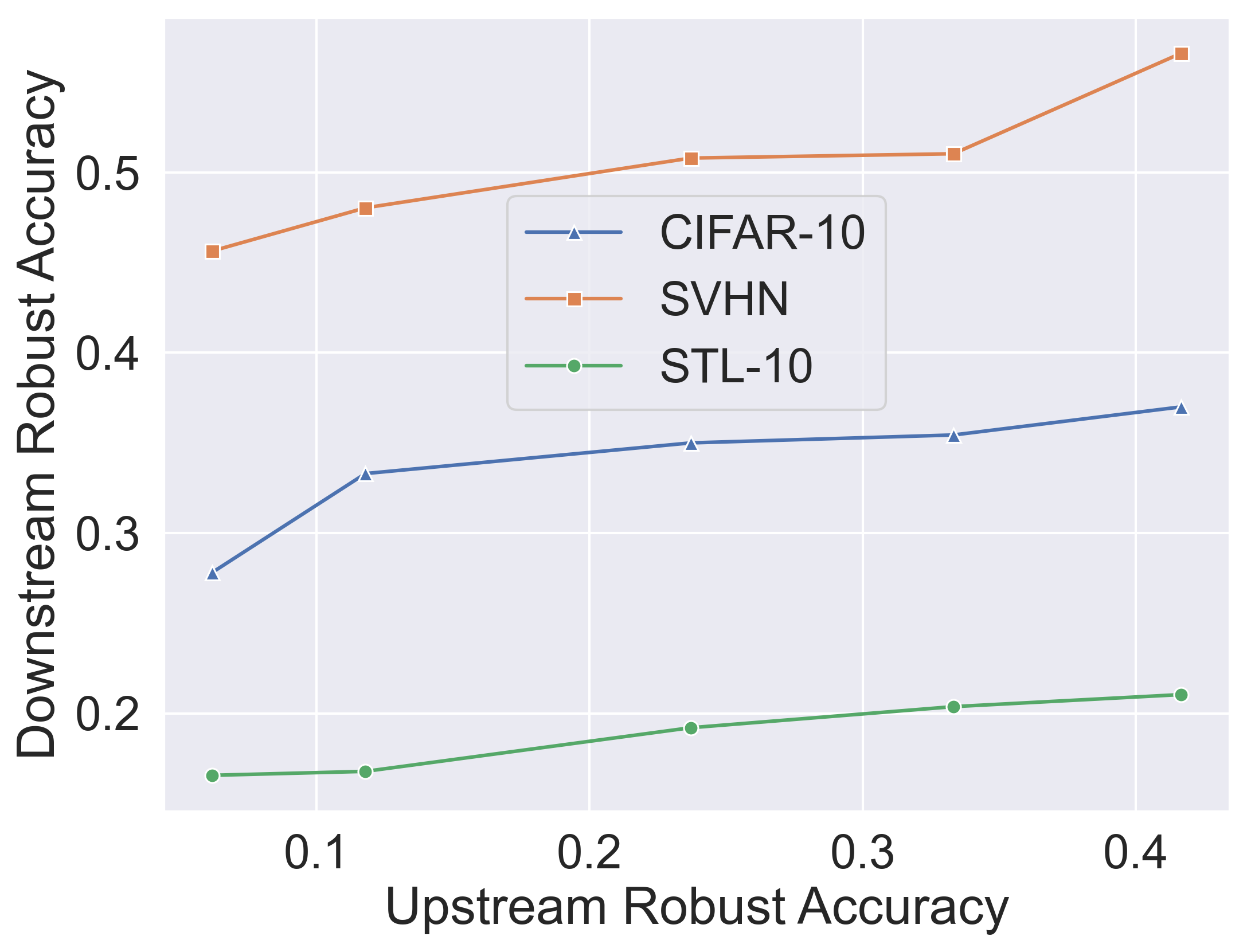}}
\caption{The relationship between the parameter robustness of the pretrained models and that of the models fine-tuned on the downstream tasks.}
\label{fig:robust_transfer}
\end{figure}

We further explore the transferability of parameter robustness in the context of transfer learning. In specific, we used CIFAR-100 as the upstream training dataset, and employed the CIFAR-10, SVHN, and STL-10 datasets, respectively for downstream training. We first trained the model on the upstream data with five different intensities of Adversarial Model Perturbation (AMP) characterized by the values \(\gamma = 0, 0.1, 0.2, 0.5, \) and \(1.0\). This procedure yielded five sets of pre-trained models, each with distinct robust accuracies.
Subsequently, we performed transfer learning from each of these five distinct pre-trained models on the downstream data. Upon the completion of training, the robust accuracy of models on the downstream dataset was computed. ~\autoref{fig:robust_transfer} shows the high between the robust accuracies of the pre-trained models (x-axis) and the robust accuracy of downstream models (y-axis). In general, it can be observed that the pre-trained models with higher robust accuracy tend to maintain their enhanced robust accuracy during the transfer learning for the downstream tasks. This observation underscores the inherent transferability of parameter robustness in a transfer learning context.

\para{Design Goals}. Based on the results from the toy examples as illustrated above, our objective is to build a pre-trained model that achieves two primary goals: robustness and effectiveness, which we elaborate below.

\vspace{2pt}\noindent$\bullet$\textit{~Robustness}.
As observed earlier, we aim for our pre-trained model to exhibit strong parameter robustness (for instance, a smaller value of $\rho$). This ensures that when the downstream model employs the DPML algorithm, the influence on its performance is minimal.

\vspace{2pt}\noindent$\bullet$\textit{~Generality}.
While striving for pronounced parameter robustness in the pre-trained model, it is imperative that the model's inherent performance remains commendable. Should we solely emphasize parameter robustness, the model may exhibit an extremely low value of $\rho$—indicating high robustness—but low accuracy. For instance, $f(x;\theta +\Delta) = f(x;\theta) \neq y$. As illustrated in \autoref{fig:gamma_amp}, when $\gamma$ exceeds 1, the pre-trained model's performance may deteriorate considerably. As a result, the DPML algorithm employed for downstream training might not inherit sufficient effective features from the pre-trained model, leading to a diminished outcome. Hence, it is equally pivotal for the pre-trained model to retain high performance as compared to enhance its parameter robustness.

\section{\ourmethod{}}
\label{sec:method}

In this section, we present the design and implementation of the proposed \ourmethod{}, together with a theoretical analysis of its effectiveness. 

In essence, the design of \ourmethod{} is rooted in the two observations mentioned earlier. Our objective is for the pre-trained model to achieve higher parameter robustness while maintaining its performance. Given the observation that parameter robustness can be transferred when the model is passed downstream, the fine-tuning DPML algorithms can operate under the same privacy budget. This, in turn, minimizes the impact on performance.

\subsection{Design Challenge}
\label{subsec:challenge}

We observe that while AMP can enhance parameter robustness, leading to an improvement in the DP-SGD performance of downstream fine-tuned models, the extent of this robustness enhancement is quite limited. This, in turn, results in only a marginal performance boost. As we illustrated in \autoref{subsec:compare} using AMP as an example, the effect of existing model parameter adversarial training on enhancing the parameter robustness of pre-trained models is not significant. Consequently, the impact on the performance of downstream DPML algorithms is also marginal. This is primarily because the existing methods aimed at improving parameter robustness were primarily designed to enhance model generalization. Since they were not explicitly designed for robustness enhancement, the gains in parameter robustness tend to be limited. While these methods can indeed be adjusted to prioritize robustness by increasing the perturbation magnitude at the expense of generalization, such a trade-off is often costly. As depicted in \autoref{fig:gamma_amp}, a high perturbation magnitude not only fails to bring much-added robustness but also makes the model challenging to converge. Given these challenges, we introduce~\ourmethod{}, a novel approach to help pre-trained models achieve higher parameter robustness and minimize the adverse impact on performance when downstream models employ DPML algorithms.

\begin{figure}[!tbp]
\centerline{\includegraphics[width=\linewidth]{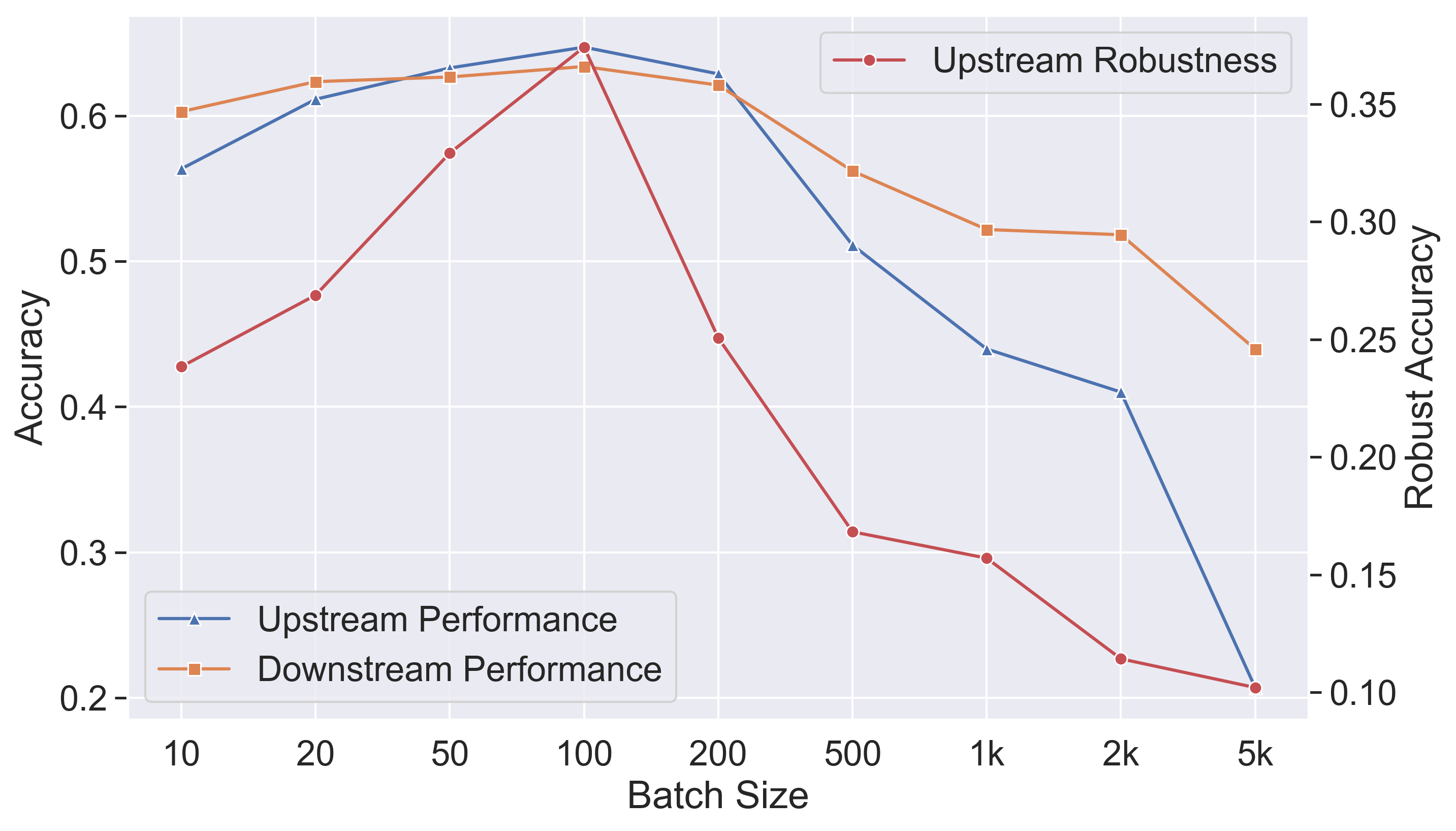}}
\caption{Impact of different batch sizes on AMP.}
\label{fig:batch_size}
\end{figure}

\subsection{Design Intuition}
\label{subsec:intuition}
As highlighted in the ``Challenge'' section above, traditional methods of enhancing parameter robustness often result in a significant trade-off between parameter robustness and model performance. We've observed that these conventional methods of improving parameter robustness primarily leverage schemes akin to SGD for adversarial training. During a single optimization iteration, the batch used to compute the perturbation and the batch used to calculate the gradient direction for optimization are the same. However, the rationale behind this approach has not been extensively discussed. In fact, taking the AMP as an example, as depicted in \autoref{fig:batch_size}, when a large batch size is chosen, the model's training becomes sub-optimal due to the utilization of this large batch. On the contrary, if the batch size is too small, computing the perturbation based on such limited data leads to inaccurate perturbations, subsequently compromising the robustness of the training. Note that a recent study suggests employing a linear scaling rule to adjust the learning rate based on the modified batch size~\cite{lr_bs}; specifically, the learning rate should be multiplied by $k$ when using a batch size of $kN$. Consequently, in this experiment, we also adhere to this setting to approximate the optimal learning rate for each batch size.

Based on these observations, we decided to decouple the batches used for computing perturbations and those for calculating parameter gradients during the optimization process. For computing perturbations, we chose a batch size larger than what is typically required for regular training. The reasoning is that perturbation computations don't necessarily benefit from being stochastic. Ideally, the entire dataset would be used to compute the perturbation. However, due to computational resource constraints, using the complete dataset for this purpose is challenging. Hence, we opt for a compromise by selecting a relatively larger batch size for the perturbation computations. In contrast, for batches used to compute model optimization gradients, we chose sizes that are conducive to model optimization, ensuring we avoid sizes that are excessively large or small that could impair model performance.

\begin{figure*}[!tbp]
\centerline{\includegraphics[width=0.8\linewidth]{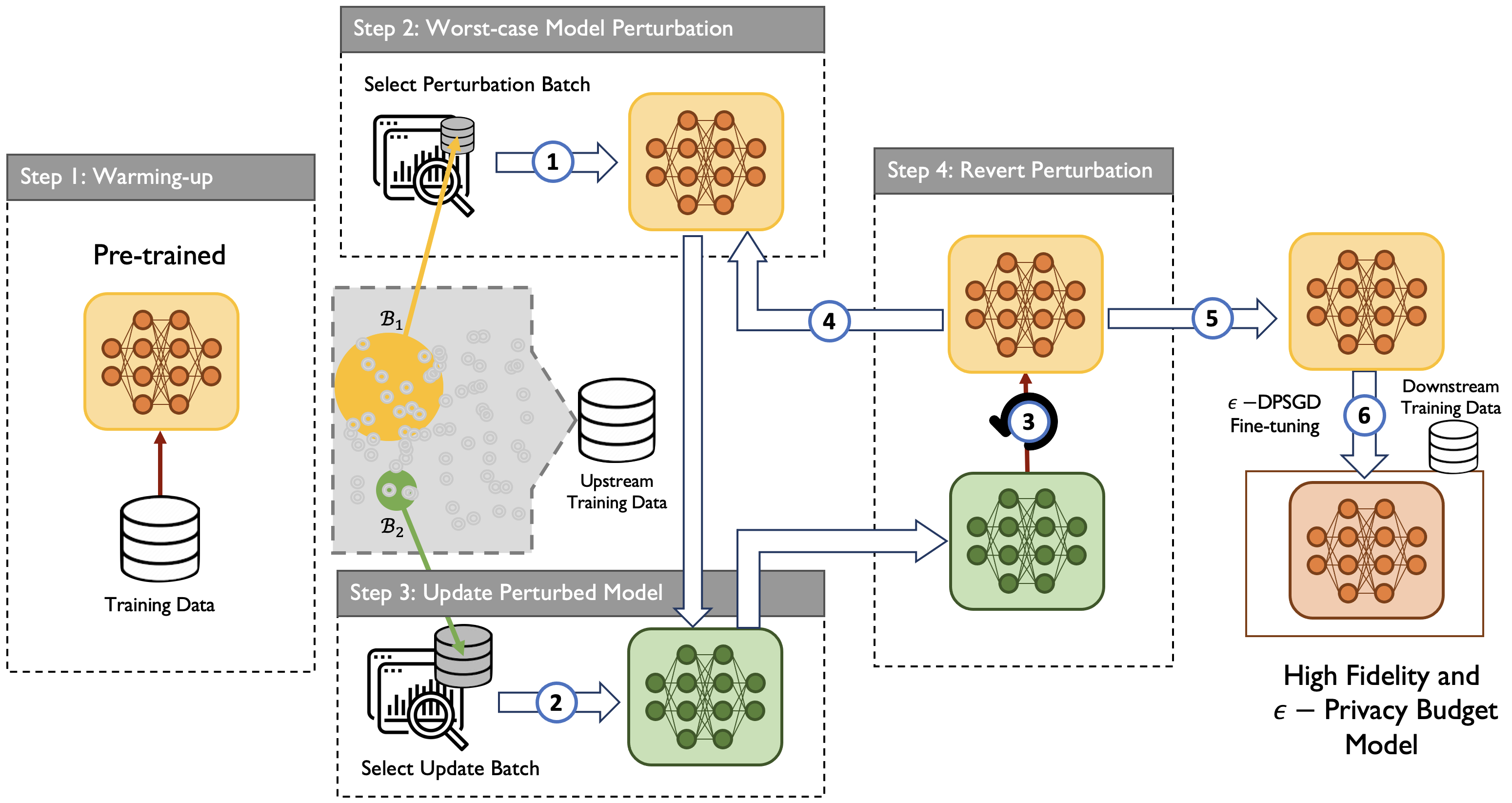}}
\caption{\small Overview of the \ourmethod{} approach. In the first step, known as standard training, we use the training data to conduct standard training on the model. This is analogous to the preheating process common in typical Adversarial Training (AT), where the model's accuracy is first brought up to a standard level, resulting in the model \(f_{\theta}\). 
Next, in step 2, We selected a batch data $\mathcal{B}_1$ (big batch size) from the training data, and utilize it to compute the worst-case model perturbation by ~\autoref{eq:perturbation} (\textcircled{1}), producing the model \(f_{\theta+\Delta}\).
In step 3, we use another batch of data $\mathcal{B}_2$ to do sgd on the model \(f_{\theta+\Delta}\) (\textcircled{2}), resulting in the model \(f_{\mathbf{w}+\Delta+\alpha}\).
In step 4, we reverse the perturbation made in step 2 within the model \(f_{\theta+\Delta+\alpha}\) (\textcircled{3}), obtaining the model \(f_{\theta+\alpha}\). 
Steps 2 to 4 are performed iteratively (\textcircled{4}), and upon reaching the desired number of iterations, we return the target model (\textcircled{5}). This culminating model is tailored to facilitate subsequent fine-tuning processes, enabling models built upon it to more efficaciously deploy the DPML optimization algorithm(\textcircled{6}). This ensures not only the protection of the privacy of the fine-tuned dataset but also a notable enhancement in DPML performance.}
\label{fig:overview}
\end{figure*}

\subsection{Methodology Overview}
\label{subsec:method}

\autoref{fig:overview} illustrates the overall workflow of \ourmethod{}. To efficiently achieve the parameter robustness of pre-trained models, ~\ourmethod{} employs a Min-Max optimization, iterating over four steps of model updating as presented below.

\para{Step 1: Warming-up.}
The pre-trained model initiates with a warming-up phase using standard training techniques to ensure the accuracy (ACC) meets the targeted benchmark. We refer to the model post-warming-up as \(f_{\mathbf{\theta}}\).

\para{Step 2: Worst-case Model Perturbation.}
We select an enlarged batch of training data, denoted as \(\mathcal{B}_1\). We then compute the direction of the gradient that maximizes the loss of this batch under the current model parameters. Subsequently, we perturb the model parameters along this direction with the intention of impairing the model's predictive capability on this large batch of data when subjected to such perturbation. We denote the perturbation as $\Delta$, and the model after perturbation as $f_{\mathbf{\theta}+\Delta}$.

\para{Step 3: Update on the Perturbed Model.}
To enhance the model's resilience to perturbations in its parameters (parameter robustness), we optimize the perturbed model with the goal of maintaining good performance even when the model is subjected to such perturbations. We select a standard-sized batch of training data, represented as \(\mathcal{B}_2\). We then perform regular training using stochastic gradient descent (SGD) on \(\mathcal{B}_2\). The updated model is denoted as \(f_{\mathbf{\theta}+\Delta + \alpha}\), where \(\alpha\) represents the weight update resulting from the SGD.

\para{Step 4: Revert Perturbation.}
In Step 3, we have already enhanced the model's resilience to parameter perturbations. As a result, there's no longer a need to retain the noise applied to the parameters in Step 2, as doing so would only degrade our model's accuracy. Therefore, in this step, we revert the update of the model weight made in Step 2: the final model is represented as \(f_{\mathbf{\theta}+\Delta + \alpha-\Delta} = f_{\mathbf{\theta} + \alpha}\).

Our workflow operates iteratively. Upon completion of Step 4, the process loops back to Step 1. In each iteration, distinct batches of training data are randomly selected for both \(\mathcal{B}_1\) and \(\mathcal{B}_2\). The algorithm takes the total number of iterations as an input, ensuring that training concludes after the specified number of iterations.

\para{Remark.}
It is important to highlight the distinction between our approach and conventional techniques aimed at enhancing a model's generality by improving parameter robustness, such as AMP~\cite{AMP}. Specifically, AMP optimizes using the same \(\mathcal{B}_2\) and \(\mathcal{B}_1\).

The key innovation in our method lies in the strategic batch selection during two different optimization processes: the worst-case perturbation computation and the SGD gradient computation. By ensuring that the batch size for \(\mathcal{B}_1\) is relatively large, we obtain a more precise and general perturbation. In contrast, a standard batch size is used for \(\mathcal{B}_2\) to facilitate optimal training with SGD. Such nuanced improvements enable our pretrained model, \ourmethod{}, to achieve heightened parameter robustness, resulting in significant enhancements when downstream DPSGD is applied.

\subsection{Design Details}
We now delve into a detailed presentation of the \ourmethod{} algorithm, focusing primarily on Steps 2 and 3. These steps are comparatively more complex: while Step 1 involves standard training and Step 4 simply involves a subtraction to remove noise from the parameters. Notably, it's Steps 2 and 3 that distinctly set \ourmethod{} apart from other SAM algorithms. 

First, like any typical SAM algorithm, we define a {\em norm ball} ${\rm\bf B}(\boldsymbol{\mu};\gamma)$, as a region around a given point $\boldsymbol{\mu}$ in the parameter space $\Theta$ (i.e., $\boldsymbol{\mu}\in\Theta$) with the radius of $\gamma$ ($\gamma \ge 0$):
\begin{equation}
{\rm\bf B}(\boldsymbol{\mu};\gamma):=\{\boldsymbol{\theta}\in\Theta:\Vert\boldsymbol{\theta}-\boldsymbol{\mu}\Vert\le\gamma\}
\end{equation}

Recall that a typical SAM loss is designed upon the norm ball  ${\rm\bf B}(\theta;\gamma)$:
\begin{equation}\label{eq:L_sam}
\mathcal{L}_\mathrm{SAM}(\boldsymbol{\theta}):=\max_{\Delta\in{\rm\bf B}(\boldsymbol{0};\gamma)}\frac{1}{|\mathcal{D}|}\sum_{(\boldsymbol{x},\boldsymbol{y})\in\mathcal{D}}\ell(\boldsymbol{x},\boldsymbol{y};\boldsymbol{\theta}+\Delta)
\end{equation}
where $\mathcal{D}$ represent the training dataset. The training loss, denoted by \(\ell\), varies based on the task at hand: for classification tasks, it is typically the cross-entropy, while for regression tasks, it is often the least squares.
The perturbation \(\Delta\) is the worst-case model perturbation

\para{Worst-case Model Perturbation.}
\noindent In this section, we detail the computation of \(\Delta\) used to obtain \(f_{\theta + \Delta}\) in Step 2. The worst-case model perturbation, \(\Delta\), is strategically designed to induce the model parameters \(\theta\) to achieve the fastest reduction in loss on the dataset:
\begin{equation}
\label{eq:perturbation}
    \Delta_{\mathcal{B}_1} = \argmax_{\Delta\in{\rm\bf B}(\boldsymbol{\theta};\gamma)}\frac{1}{|\mathcal{B}_1|}\sum_{(\boldsymbol{x},\boldsymbol{y})\in\mathcal{B}_1}\ell(\boldsymbol{x},\boldsymbol{y};\boldsymbol{\theta}+\Delta)
\end{equation}

Recall that in \ourmethod{}, $\mathcal{B}_1$ denote the batch of data use as to compute the worst-case model perturbation. \(\gamma\) denotes the radius of the norm ball and acts as a hyperparameter that governs the degree of parameter robustness training. The robustness of the model's parameters is directly influenced by this radius: a larger \(\gamma\) implies that the model can tolerate more noise in the parameter space, as determined by \ourmethod{}. Consequently, a greater \(\gamma\) results in a model with enhanced parameter robustness.

\para{Update Perturbed Model.}
 In this section, we detail the computation of \(\alpha\) used to obtain \(f_{\theta + \Delta + \alpha}\) in Step 3. Since in practice, the optimization is typically carried out using mini-batches. In this case, the $\mathcal{L}_\mathrm{DPAdapter}$ can be approximated using a mini-batch $\mathcal{B}_2$:
\begin{align}
\label{eq:update}
\mathcal{L}_\mathrm{~\ourmethod{}}(\boldsymbol{\theta})
&\approx\max_{\Delta_{\mathcal{B}_1}\in{\rm\bf B}(\boldsymbol{0};\gamma)}\frac{1}{|\mathcal{B}_2|}\sum_{(\boldsymbol{x},\boldsymbol{y})\in\mathcal{B}_2}\ell(\boldsymbol{x},\boldsymbol{y};\boldsymbol{\theta}+\Delta_{\mathcal{B}_1})\nonumber\\
&:=\mathcal{J}_{\mathrm{~\ourmethod{}},\mathcal{B}_1,\mathcal{B}_2}(\boldsymbol{\theta})
\end{align}

Hence, at this point, the model's update is given by 
\begin{equation}
    \alpha = \eta_2 \nabla \mathcal{J}_{\mathrm{~\ourmethod{}}, \mathcal{B}_1, \mathcal{B}_2}(\boldsymbol{\theta})
\end{equation}

where \(\eta_2\) represents the optimization step size. It is crucial to highlight the primary distinction between \ourmethod{} and other SAM algorithms, which lies in \autoref{eq:perturbation} and \autoref{eq:update}. In other SAM algorithms, such as those cited in ~\cite{AMP,AWP,sam}, the batch used to compute the worst-case perturbation, denoted as \(\mathcal{B}_1\), is the same as the batch utilized for computing the update. We argue that by separately selecting \(\mathcal{B}_1\) and \(\mathcal{B}_2\), the trained model can achieve a balance between parameter robustness and overall model performance in \autoref{subsec:compare}.

The specifics of the \ourmethod{} algorithm can be found in \autoref{alg:DPAdapter}. Importantly, in alignment with other SAM algorithms such as~\cite{AMP,AWP,sam}, the finalized model utilizes the parameters denoted as $\boldsymbol{\theta}^\ast_\mathrm{DPAdapter}$. These parameters are employed for model inference (or prediction) without the need for additional perturbations.

As mentioned earlier, the training process of \ourmethod{} is divided into two primary stages: The perturbation computation and the update computation. 
 
\vspace{2pt}\noindent$\bullet$\textbf{~Perturbation Computation}.
Line~\ref{algline:cal_perturb_start} to Line~\ref{algline:cal_perturb_end} in \autoref{alg:DPAdapter}. This phase is responsible for calculating the worst-case perturbation, $\Delta_{\mathcal{B}_1}$, based on the prevailing conditions.

\vspace{2pt}\noindent$\bullet$\textbf{~Update Computation}.
Line~\ref{algline:cal_update_start} to Line~\ref{algline:cal_update_end} in \autoref{alg:DPAdapter}. In this phase, the $\Delta_{\mathcal{B}_1}$ derived from the previous perturbation computation step is used to determine the parameter optimization gradient, $\nabla\mathcal{J}_{\mathrm{~\ourmethod{}},\mathcal{B}_1,\mathcal{B}_2}$. Following this, the model parameters are updated.

Training can be executed for a specified number of epochs $K$, or it can continue until the parameter loss stabilizes, which indicates the completion of the training process.

\begin{algorithm}[!tbp]
\small
\caption{\ourmethod{} Training}
\label{alg:DPAdapter}
\begin{algorithmic}[1]
\REQUIRE Training set $\mathcal{D}=\{(\boldsymbol{x},\boldsymbol{y})\}$, perturbation batch size $m_1$, update batch size $m_2$, loss function $\ell$, perturbation computation learning rate $\eta_1$, update learning rate $\eta_2$,
norm ball radius $\gamma$, number of update iteration $K$.
\WHILE{$k<K$}
\STATE Draw $\mathcal{B}_1=\{(\boldsymbol{x}_i,\boldsymbol{y}_i)\}_{i=1}^{m_1}$ from training set $\mathcal{D}$
\STATE Draw $\mathcal{B}_2=\{(\boldsymbol{x}_i,\boldsymbol{y}_i)\}_{i=1}^{m_2}$ from training set $\mathcal{D}$
\STATE Set perturbation: $\Delta_{\mathcal{B}_1}\gets\boldsymbol{0}$
\STATE \label{algline:cal_perturb_start} Calculate the perturbation: \\  $\qquad\Delta_{\mathcal{B}_1}\gets \eta_1\argmax_{\Delta\in{\rm\bf B}(\boldsymbol{\theta_k};\gamma)}\frac{1}{|m_1|}\sum_{(\boldsymbol{x},\boldsymbol{y})\in\mathcal{B}_1}\ell(\boldsymbol{x},\boldsymbol{y};\boldsymbol{\theta_k}+\Delta)$
\IF{$\Vert\Delta_{\mathcal{B}_1}\Vert_2>\gamma$}
\STATE Normalize perturbation: $\Delta_{\mathcal{B}_1}\gets\gamma\Delta_{\mathcal{B}_1}/\Vert\Delta_{\mathcal{B}_1}\Vert_2$
\ENDIF
\STATE \label{algline:cal_perturb_end} Add the perturbation to parameters: \\ $\qquad\boldsymbol{\theta}_{k}\gets\boldsymbol{\theta}_k + \Delta_{\mathcal{B}_1}$
\STATE \label{algline:cal_update_start} Compute update gradient: \\ $\qquad\nabla\mathcal{J}_{\mathrm{~\ourmethod{}},{\mathcal{B}_1},{\mathcal{B}_2}}\gets\sum_{i=1}^{|\mathcal{B}_2|}\nabla_{\boldsymbol{\theta}}\ell(\boldsymbol{x}_i,\boldsymbol{y}_i;\boldsymbol{\theta}_k)/{m_2}$
\STATE Update on the perturbed model : \\ $\qquad\boldsymbol{\theta}_{k+1}\gets\boldsymbol{\theta}_k-\eta_2\nabla\mathcal{J}_{\mathrm{~\ourmethod{}},\mathcal{B}_1,\mathcal{B}_2}$

\STATE \label{algline:cal_update_end} Revert perturbation: $\boldsymbol{\theta}_{k+1}\gets\boldsymbol{\theta}_{k+1}-\Delta_{\mathcal{B}_1}$

\STATE Increment iteration: $k\gets k+1$

\ENDWHILE
\end{algorithmic}
\end{algorithm}

\subsection{Theoretical Analyses}
\label{subsec:theory}

\newtheorem{assumption}{Assumption}
\newcommand{\xu}{X_{\text{update batch}}}
\newcommand{\xp}{X_{\text{perturbation batch}}}
\newcommand{\cB}{\mathcal{B}}
\newcommand{\cS}{\mathcal{S}}
\newcommand{\cD}{\mathcal{D}}
\newcommand{\wb}{\mathbf{w}}
\newcommand{\btheta}{\boldsymbol{\theta}}

In this section, we aim to provide a theoretical explanation of why ~\ourmethod{} achieves improvements in the DPML algorithm. Specifically, our theoretical examination revolves around two key points:

\begin{enumerate}
    \item Understanding the rationale behind the separate selection of $\cB_1$ and $\cB_2$, particularly emphasizing why the batch size of $\cB_2$ should be large.
    \item Elucidating how enhancing the robustness of pretrained parameters can ensure superior performance in downstream training, given the same privacy budget.
\end{enumerate}

The first point is theoretically proven in Theorem 1, while the second is validated in Theorem 2.

To give a basic idea, through Theorem 1, we aim to theoretically demonstrate that, upon selecting an appropriate \({\mathcal{B}_1}\) conducive to optimization and holding it constant, we can bound the difference between the actual loss and the maximum lower bound (Infimum) of the loss. Essentially, as the magnitude of \({\mathcal{B}_2}\) increases, the upper bound on the expected convergence rate of the model obtained through ~\ourmethod{} also increases.

In theorem 2, we prove that given a model training by DP-SGD, the expected difference between the model's loss and the maximum lower bound of the achievable loss (termed as DP-SGD's performance) is positively proportional to $\rho$. Specifically, a lower value of $\rho$ (indicating stronger parameter robustness) in the model initial stage, results in a lower loss when optimized using DP-SGD.

Next, we delineate the setup for our theoretical framework. We first analyze the convergence behavior of our proposed algorithm. We show that by using a different update batch (where $|\cB_1| \neq |\cB_2|$). We summarize our algorithm as follows. Let $L_i(\btheta):=\ell(f_{\btheta}(\boldsymbol{x}_i), \boldsymbol{y}_i)$. For any batch of the data $\cS \subseteq \cD$, let $L_{\cS}$ be the average of $L_i, i\in \cS$. Then our algorithm is as follows. Starting from the initial parameter $\btheta_0$, for each iteration $t$, we have
\begin{align}
\btheta_{t+1} = \btheta_t - \eta_2\cdot \nabla L_{\cB_2}(\btheta_t + \Delta_t),\ 
\Delta_t:=\eta_1\cdot \nabla L_{\cB_1}(\btheta_t).\label{alg:111}
\end{align}
\autoref{alg:111} is essentially the same as \autoref{alg:DPAdapter} in \autoref{subsec:method} without the perturbation normalization step (line 6-8 in \autoref{alg:DPAdapter}). 
Meanwhile, we need the following assumptions. 
\begin{assumption}
$f$ is $\beta$-smooth w.r.t. $\btheta$, i.e., $\|\nabla f_{\btheta}(\boldsymbol{x}) - \nabla f_{\btheta'}(\boldsymbol{x})\|_2 \leq \beta\|\btheta - \btheta'\|_2$. 
\end{assumption}
\begin{assumption}
    $\ell$ is $\beta_1$-Lipschitz continuous w.r.t. $\boldsymbol{x}$ i.e., $|\ell(\boldsymbol{x},\boldsymbol{y}) - \ell(\boldsymbol{x}',\boldsymbol{y})| \leq \beta_1|\boldsymbol{x} - \boldsymbol{x}'|$.
\end{assumption}
\begin{assumption}
    $\ell$ is $\beta_2$-smooth w.r.t. $\boldsymbol{x}$, i.e., $|\nabla_{\boldsymbol{x}}\ell(\boldsymbol{x},\boldsymbol{y}) - \nabla_{\boldsymbol{x}}\ell(\boldsymbol{x}',\boldsymbol{y})| \leq \beta_2|\boldsymbol{x} - \boldsymbol{x}'|$.
\end{assumption}
\begin{assumption}
    $\nabla L_i$ is $\hat\sigma^2$-variance bounded gradient, i.e., $\mathbb{E}_i\|\nabla L_i(\btheta) - \nabla L_\cD(\btheta)\|_2^2 \leq \hat\sigma^2$. Here $\mathbb{E}_i$ denotes the expectation over $i\in1,\dots, n$. 
\end{assumption}
\begin{assumption}
    $L_\cD$ satisfies the $\mu$-PL-condition, i.e., $\|\nabla L_\cD(\btheta)\|_2^2 \geq 1/\mu\cdot(L_\cD(\btheta) - \inf_{\hat \btheta} L_\cD(\hat \btheta))$.
\end{assumption}
We also formally define the parameter robustness in \autoref{eq:para robust} as follows:

\begin{equation}
    \rho(f) = \max_{\boldsymbol{x}, \btheta, \Delta}|f_{\theta+\Delta}(\boldsymbol{x}) - f_{\theta}(\boldsymbol{x})|/\|\Delta\|_2.\notag
\end{equation}
For simplicity, we use $\rho:=\rho(f)$ with a slight abuse of notation. Then we have the following theorem.
\begin{theorem}[Informal]\label{thm:1}
With proper selection of parameters, our algorithm enjoys the following gradient norm bound:
    \begin{align}
        &\frac{1}{T}\sum_{t=1}^T \mathbb{E}(L_\cD(\btheta_t) - \inf_{\hat \btheta} L_\cD(\hat \btheta) )\notag \\
        &\leq \mu\cdot\bigg(\frac{L_\cD(\btheta_0)}{T} + \frac{\hat\sigma^2}{16\hat\beta}(1/|\cB_2| + 1/|\cB_1|)\bigg),\notag
    \end{align}
    where $\hat\beta = (\rho^2\beta_2 + \beta\beta_1)$. 
\end{theorem}
The detailed theorem and proof are deferred to Appendix \ref{app:thm1}. 
\autoref{thm:1} suggests that selecting a large perturbation batch size $|\cB_1|$ makes our algorithm have a better convergence rate, indicating better performance while using a large batch of perturbations $|\cB_1|$.

We have another analysis of the noise tolerance and the DP-SGD performance. 
Then we have our theorem. 
\begin{theorem}[Informal]\label{thm:2}
    Given $(\epsilon, \delta)$, with proper selection of parameters, let $\btheta_{\text{out}}$ be the final output of DP-SGD, then DP-SGD is $(\epsilon, \delta)$-DP and enjoys the following utility bound:
    \begin{align}
        \mathbb{E}(L_\cD(\btheta_{\text{out}}) - \inf_{\btheta} L_\cD(\btheta) )\leq C\cdot \frac{\rho\sqrt{\rho^2\beta_2 + \beta\beta_1}}{|\cD|\epsilon},\notag
    \end{align}
    where $C = c\cdot\mu\beta_1\sqrt{d\log(|\cD|/\delta)\log(1/\delta)L_\cD(\btheta_0)}$, $c$ is some positive constant, $d$ is the dimension of parameter $\btheta$. 
\end{theorem}
The detailed algorithm, theorem, and proof are deferred to Appendix~\ref{app:thm2}. 
\autoref{thm:2} suggests that the parameter robustness $\rho$ indeed affects the final utility. 

\section{Evaluation}
\label{sec:eval}
\subsection{Experimental Setup}
\label{subsec:setup}
\para{Datasets}. We use the following five datasets in our evaluation.

\vspace{4pt}\noindent$\bullet$\textit{~CIFAR-10}~\cite{krizhevsky2009learning}: This dataset contains $50,000$ training images and $10,000$ testing images. Each image has a size of $32\times32\times3$ and belongs to one of $10$ classes.

\vspace{4pt}\noindent$\bullet$\textit{~CIFAR-100}~\cite{krizhevsky2009learning}: This dataset contains $50,000$ training images and $10,000$ testing images. Each image has a size of $32\times32\times3$ and belongs to one of $100$ classes.

\vspace{4pt}\noindent$\bullet$\textit{~SVHN}~\cite{svhn}: In this dataset, each image represents a digit from the house numbers in Google Street View. The size of each image is $32\times32\times3$. Each image belongs to one of the $10$ digits. This dataset has $73,257$ training images and $26,032$ testing images.

\vspace{4pt}\noindent$\bullet$\textit{~STL-10}~\cite{STL-10}: This dataset contains $5,000$ training images and $8,000$ testing images. Each image has a size of $96\times96\times3$ and belongs to one of $10$ classes.

In this paper, CIFAR-100 serves as a public dataset used for pre-training, while CIFAR-10, SVHN, and STL-10 serves as the private data/tasks.

\para{Pre-training Procedure}.
In our experiments, we utilize CIFAR-100 as our pre-training dataset and employ \ourmethod{} to train a ResNet20 model~\cite{ResNet} as the pre-trained model. We compute the mean and standard deviation on the training set to normalize the input images. We adopt cross-entropy as the loss function $\ell$ and utilize the SGD optimizer with momentum, incorporating a step-wise learning rate decay. Specifically, the model is trained for $K=1,000$ epochs, with the outer learning rate $\eta_2$ initialized at $0.1$ and divided by $10$ after $500$ and $750$ epochs. The momentum is set to $0.9$, and the weight decay is set to $\num{1e-4}$. The inner batch size $m_1$ is set to $5,000$, with the inner learning rate $\eta_1$ set to $1.0$; the outer batch size $m_2$ is set to $50$. For perturbation magnitude, we adopt $\gamma=2.0$ by default.

\para{Fine-tuning Procedure}. 
Given a pre-trained model, we utilize it to train private downstream classifiers for the remaining three datasets: CIFAR-10, SVHN, and STL-10. We employ the parameters of the pre-trained model as the initial parameters of the downstream classifier, which is then trained on the downstream dataset.

We use R\'enyi DP to accumulate the overall privacy budget and precompute the required noise scale ($\sigma$ in DP-SGD) numerically~\cite{DPSGD, Gaussian_Mechanism}. We maintain $\delta=10^{-5}$ and utilize different privacy budgets: $\epsilon = \{1, 4\}$. The clipping threshold for all algorithms is fixed at 4, except when an algorithm employs special clipping strategies. The cross-entropy loss function and the DP-SGD optimizer with momentum are adopted when training a downstream classifier. The model is fine-tuned for $100$ epochs, with the learning rate initialized at $0.01$ and momentum set to $0.9$.

\para{DPML Algorithms.}
According to the different optimization methods that the downstream model could adopt, we consider three types of DP algorithms in addition to vanilla DP-SGD:

\vspace{4pt}\noindent$\bullet$\textit{~GEP}~\cite{yu2021not}:
Yu et al. observed that in vanilla DP-SGD, the amount of noise increases with the model size and proposed a solution, GEP~\cite{yu2021not}, to reduce the gradient dimension before adding noise. GEP first calculates an anchor subspace, which contains some gradients of public data, using the power method. Subsequently, it projects the gradient of private data into the anchor subspace, resulting in a low-dimensional gradient embedding and a small-norm residual gradient. These two components are independently processed with the DP mechanism and then combined to update the original weight.

\vspace{4pt}\noindent$\bullet$\textit{~AdpAlloc}~\cite{yu2019differentially}:
It proposes a dynamic noise-adding mechanism, eschewing the practice of maintaining a constant noise multiplier $\sigma$ throughout every training epoch in vanilla DP-SGD. Instead of utilizing a static variance in the Gaussian mechanism, it replaces it with a function of the epoch:
$\mathit{M}(d) = \mathit{f}(d) + \mathbf{N}(0, S_\mathit{f}^2 \cdot \sigma_t^2 )$,
where the value of $\sigma_t$ is contingent upon the final privacy budget, epoch, and schedule function. The schedule function delineates the adjustment of the noise scale throughout training. Yu et al. proposed several pre-defined schedules. For our evaluation, we select \textit{Exponential Decay}, which demonstrated the best average performance in \cite{yu2019differentially}. The mathematical expression for \textit{Exponential Decay} is $\sigma_t=\sigma_0 e^{-k t}$, where $k (k>0)$ represents the decay rate and $\sigma_0$ denotes the initial noise scale.

\vspace{4pt}\noindent$\bullet$\textit{~AdpClip}~\cite{andrew2021differentially}:
An adaptive clipping threshold mechanism is utilized, setting the clip threshold to a specified quantile of the update norm distribution at each epoch. Formally, the clipping threshold \(C_t\) in epoch \(t\) can be computed as $C_t = C_{t-1} \cdot \exp\left(-\eta_C(\overline{b}-\gamma )\right)$, where \(\gamma \in \left[0,1\right]\) is a quantile to be matched, \( \overline{b} \triangleq \frac{1}{m} \sum_{i \in[m]} \mathbb{I}_{x_i \leq C}\) represents the empirical fraction of samples with a value at most \(C\), and \(\eta_C\) is the learning rate with a default value of 0.2, as indicated in \cite{andrew2021differentially}. 
To circumvent the issue of \(\overline{b}\) revealing private information, the Gaussian mechanism is applied to $\overline{b}$: $\tilde{b}^t=\frac{1}{m}\left(\sum_{i \in \mathit{Q}^t} b_i^t+\mathit{N}\left(O, \sigma_b^2\right)\right)$.
This method expends a negligible privacy budget while closely tracking the quantile. While AdpClip was initially designed for federated learning~\cite{Decentralized}, it can also be used in traditional centralized learning scenarios.

\begin{table*}[!t]
\setlength{\tabcolsep}{2.5mm}{
\caption{Comparison of the performance of training from scratch, standard pre-training, vanilla SAM, and \ourmethod{} across various downstream tasks under different privacy budgets.}
\small
\label{tab:main}
\renewcommand{\arraystretch}{1.1}
\begin{center}
\begin{tabular}{cccccccccc}
\toprule
\multirow{2}{*}{\begin{tabular}[c]{@{}c@{}}DPML\\ Algorithms\end{tabular}} & \multirow{2}{*}{\begin{tabular}[c]{@{}c@{}}Upstream\\ Training Method\end{tabular}} & \multicolumn{2}{c}{CIFAR10}       & \multicolumn{2}{c}{SVHN}          & \multicolumn{2}{c}{STL10}         & \multicolumn{2}{c}{Average}       \\ \cline{3-10} 
                            &                               & $\epsilon=1$           & $\epsilon=4$           & $\epsilon=1$           & $\epsilon=4$           & $\epsilon=1$           & $\epsilon=4$           & $\epsilon=1$           & $\epsilon=4$           \\ 
\midrule
\multirow{4}{*}{DP-SGD}     & None (Scratch)~\cite{DPSGD}                       & 0.4288          & 0.5070          & 0.7194          & 0.8380          & 0.2831          & 0.3370          & 0.4771          & 0.5607          \\
                            & Standard Pre-training~\cite{Pretrain_DP}                        & 0.5928          & 0.7210          & 0.7822          & 0.8970          & 0.3282          & 0.5695          & 0.5677          & 0.7292          \\
                            & Vanilla SAM (Ours)                          & 0.6216          & 0.7650          & 0.8042          & 0.9014          & 0.3625          & 0.6212          & 0.5961          & 0.7625          \\
                            & \ourmethod{} (Ours)                   & \textbf{0.6416} & \textbf{0.7746} & \textbf{0.8058} & \textbf{0.9018} & \textbf{0.3951} & \textbf{0.6364} & \textbf{0.6142} & \textbf{0.7709} \\ 
\midrule
\multirow{4}{*}{AdpClip}    & None (Scratch)~\cite{andrew2021differentially}                       & 0.3738          & 0.5348          & 0.6258          & 0.8294          & 0.2270          & 0.3543          & 0.4089          & 0.5728          \\
                            & Standard Pre-training~\cite{Pretrain_DP}                        & 0.4198          & 0.6780          & 0.6196          & 0.8914          & 0.2406          & 0.4890          & 0.4267          & 0.6861          \\
                            & Vanilla SAM (Ours)                          & 0.5962          & 0.7108          & 0.6672          & 0.8962          & 0.2539          & 0.5446          & 0.5058          & 0.7172          \\
                            & \ourmethod{} (Ours)                   & \textbf{0.6008} & \textbf{0.7186} & \textbf{0.6730} & \textbf{0.9012} & \textbf{0.3023} & \textbf{0.5676} & \textbf{0.5254} & \textbf{0.7291} \\ 
\midrule
\multirow{4}{*}{GEP}        & None (Scratch)~\cite{yu2021not}                       & 0.4008          & 0.4672          & 0.5892          & 0.8158          & 0.2360          & 0.3239          & 0.4087          & 0.5356          \\
                            & Standard Pre-training~\cite{Pretrain_DP}                        & 0.6512          & 0.7456          & 0.8078          & 0.8500          & 0.3326          & 0.6002          & 0.5972          & 0.7319          \\
                            & Vanilla SAM (Ours)                          & 0.6808          & 0.7472          & 0.8132          & 0.8538          & 0.3739          & 0.6168          & 0.6226          & 0.7393          \\
                            & \ourmethod{} (Ours)                   & \textbf{0.6890} & \textbf{0.7692} & \textbf{0.8180} & \textbf{0.8686} & \textbf{0.4730} & \textbf{0.6462} & \textbf{0.6600} & \textbf{0.7613} \\ 
\midrule
\multirow{4}{*}{AdpAlloc}   & None (Scratch)~\cite{yu2019differentially}                       & 0.4370          & 0.5166          & 0.6248          & 0.7678          & 0.2923          & 0.3391          & 0.4514          & 0.5412          \\
                            & Standard Pre-training~\cite{Pretrain_DP}                        & 0.4506          & 0.6982          & 0.6604          & 0.8938          & 0.2946          & 0.5104          & 0.4685          & 0.7008          \\
                            & Vanilla SAM (Ours)                           & 0.5296          & 0.7372          & 0.7652          & 0.8998          & 0.2933          & 0.5675          & 0.5294          & 0.7348          \\
                            & \ourmethod{} (Ours)                   & \textbf{0.5352} & \textbf{0.7406} & \textbf{0.7862} & \textbf{0.9008} & \textbf{0.2938} & \textbf{0.6111} & \textbf{0.5384} & \textbf{0.7508} \\ 
\bottomrule
\end{tabular}
\end{center}
}
\end{table*}

\subsection{Effectiveness of \ourmethod{}}
\label{subsec:results}
In this section, we empirically validate the overall effectiveness of \ourmethod{} and conduct ablation studies to illustrate the effectiveness of each component.

\para{Setup}. We conduct experiments using four different DPML algorithms across three distinct private downstream tasks (CIFAR-10, SVHN, and STL-10), under four different pre-training settings: (i) From Scratch: utilizing randomly initialized weight values (i.e., no pre-training); (ii) Standard Pre-training: conducting pre-training on the CIFAR-100 dataset using standard training procedures; (iii) Vanilla Sharpness-Aware Minimization (SAM): pre-training with AMP~\cite{AMP}, employing the optimal perturbation magnitude $\gamma=1.0$ and the optimal batch size $m=100$; (iv) \ourmethod{}: pre-training with the proposed \ourmethod{}, using hyper-parameters described in \autoref{subsec:setup}.

\para{Observations}. \autoref{tab:main} illustrates the downstream accuracy for various settings. In general, we observe that the proposed \ourmethod{} substantially enhances the downstream accuracy in all settings. For instance, when the privacy budget is set to $\epsilon=1$ and the DPML algorithm is configured to vanilla DP-SGD, \ourmethod{} elevates the average accuracy across three downstream datasets to 61.42\%. In contrast, when utilizing models pre-trained normally, the accuracy is only 56.95\%, marking an improvement of over 4\%. When the DPML algorithm is AdpClip, \ourmethod{} boosts the accuracy from 42.67\% to 52.54\%, representing a near 10\% improvement.

We note that, compared to training from scratch, a normally pre-trained model consistently yields an improvement in downstream accuracy. This finding aligns with previous work which noted that transfer learning from public data significantly enhances the accuracy of DPML algorithms~\cite{Pretrain_DP}. Further, we note that employing vanilla SAM to provide parameter robustness can additionally enhance downstream accuracy. For example, when the privacy budget is designated as $\epsilon=1$ and the DPML algorithm is configured to AdpAlloc, using a standard pre-trained model improves the average downstream accuracy by 1.71\%. Meanwhile, applying vanilla SAM can further augment the average downstream accuracy by 6.09\% (compared with standard pre-training), achieving an additional gap that is over three times larger than the previous gap achieved by standard pre-training. This underscores the significance of leveraging parameter robustness to enhance DPML algorithms.

Moreover, we observe that employing \ourmethod{} consistently achieves higher downstream accuracy compared to using vanilla SAM. For instance, when the privacy budget is set to $\epsilon=4$ and the DPML algorithm is configured to GEP, applying vanilla SAM enhances the average downstream accuracy by 0.74\% compared with standard pre-training. Meanwhile, applying \ourmethod{} can further increase the average downstream accuracy by 2.20\% (compared with vanilla SAM). When the privacy budget is $\epsilon=1$, applying vanilla SAM improves the average downstream accuracy by 2.54\%; \ourmethod{}, in contrast, achieves an improvement of 6.28\% compared with standard pre-training. This emphasizes the importance of leveraging decoupled batches to further enhance parameter robustness.

\subsection{Impact of Perturbation Magnitude}
\label{subsec:magnitude}
Recall that \ourmethod{} fundamentally accumulates parameter robustness by introducing adversarial perturbation during the training process, the magnitude of which is determined by the $\gamma$ term. To comprehend how \ourmethod{} influences the performance of downstream tasks, we assess \ourmethod{} utilizing various $\gamma$ settings. In this experiment, we fix the privacy budget at $\epsilon=1$ and the DPML algorithm is configured to vanilla DP-SGD.

\para{Observations}. \autoref{fig:gamma} illustrates the fluctuations in upstream performance, upstream parameter robustness, and the performance of various downstream tasks when applying \ourmethod{} with different $\gamma$ values, representing the perturbation magnitude. We observe that the downstream accuracy trend largely aligns with the trend of upstream parameter robustness, while demonstrating a nearly inverse relationship with the level of upstream accuracy. This observation suggests that parameter robustness is the principal factor through which \ourmethod{} enhances downstream accuracy. Given that the trends of upstream accuracy and downstream accuracy exhibit nearly inverse patterns with changes in $\gamma$, we can essentially rule out the possibility that upstream accuracy is the predominant contributor to downstream accuracy.

Note that when $\gamma$ is further increased to $5.0$, the upstream accuracy experiences a slight increase while the performance across different downstream tasks all decline. This could be attributed to the scenario where the enhancement in parameter robustness becomes constrained and the decline in upstream accuracy becomes significant. The advantages conferred by improved parameter robustness are unable to offset the disadvantages brought about by the reduction in upstream accuracy, since upstream accuracy also contributes to downstream accuracy through enhanced generalization ability. This underscores the vital necessity of maintaining a balance between parameter robustness and generalization ability.

In this experiment, for the first time, we identify and establish a connection between parameter robustness and the performance of DPML algorithms, marking one of the key contributions of this paper.

\begin{figure}[!tbp]
\centerline{\includegraphics[width=.9\linewidth]{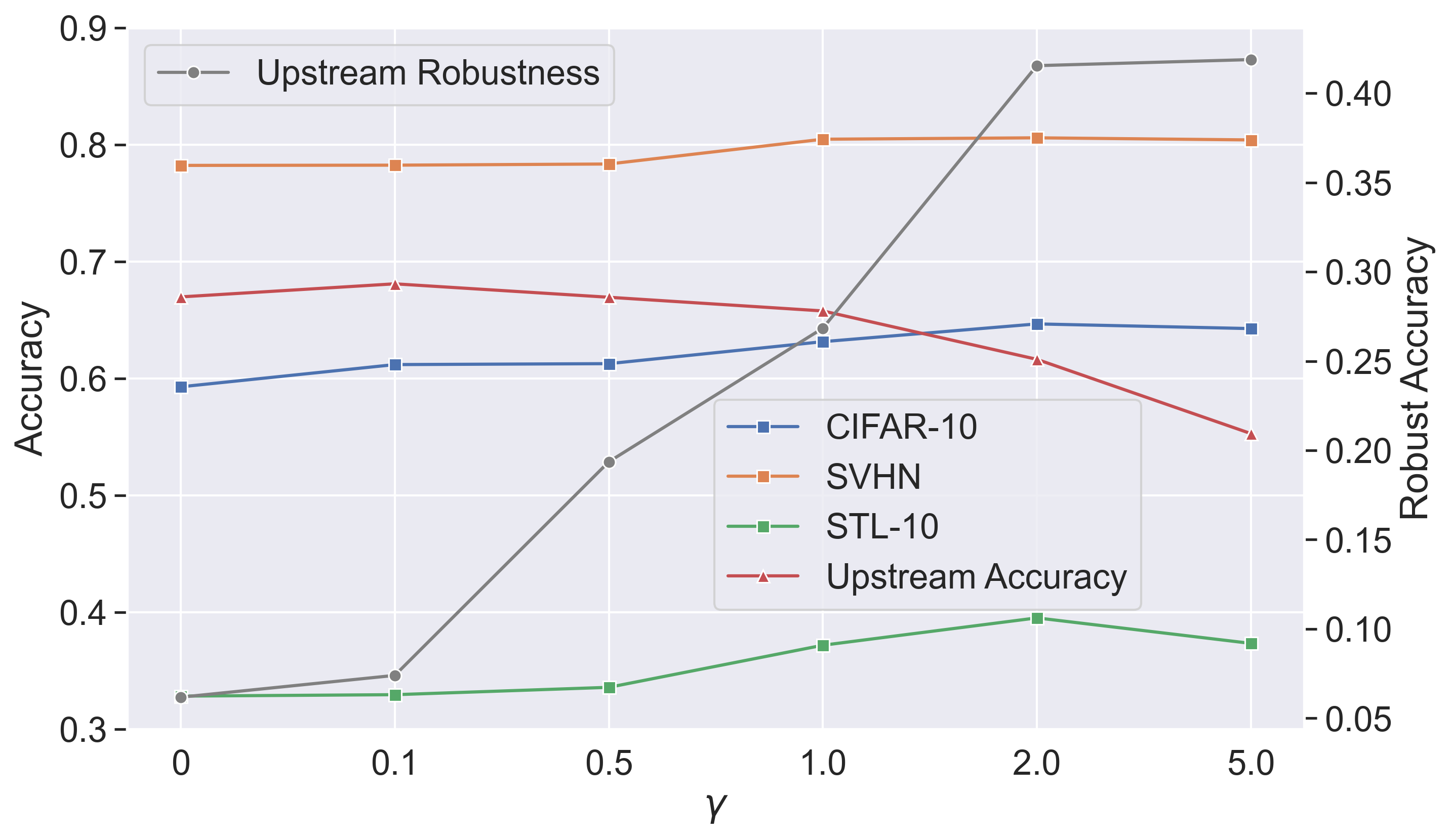}}
\caption{Impact of perturbation magnitude on \ourmethod{}.}
\label{fig:gamma}
\end{figure}

\begin{figure*}[!ht]
\centerline{\includegraphics[width=\linewidth]{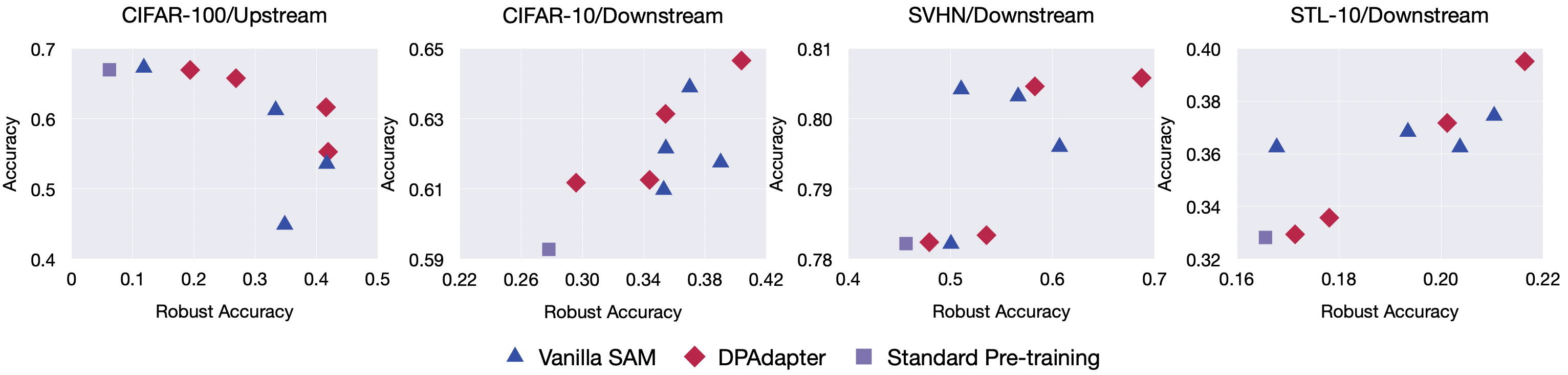}}
\caption{The results of \ourmethod{} and vanilla SAM under different perturbation magnitudes. The results of the vanilla SAM are represented by blue triangles, while the results of the proposed \ourmethod{} are depicted by red diamonds. Results obtained using a vanilla pre-trained model are illustrated by a purple rectangle. All the downstream models have the same privacy budget $\epsilon=1$.}
\vspace{2pt}
\label{fig:eps_1}
\end{figure*}

\subsection{Comparison with Vanilla SAM}
\label{subsec:compare}

\para{Setup}. In this experiment, we consider four perturbation magnitudes: 0.1, 0.5, 1.0, and 2.0. The DPML algorithm is fixed at vanilla DP-SGD and the privacy budget is fixed at $\epsilon=1$. The results of using the vanilla SAM are denoted by blue triangles, while the results of the proposed \ourmethod{} are denoted by red diamonds. The result of using a normally pre-trained model is denoted as a purple rectangle.

\para{Observations}.
\autoref{fig:eps_1} illustrate the influence of perturbation magnitude on both vanilla SAM and \ourmethod{}. The x-axis represents the robust accuracy, which measures parameter robustness, while the y-axis indicates the accuracy, reflecting model performance. We observe that models using standard pre-training typically appear in the bottom-left region, demonstrating a weak robustness and generalization trade-off, identified as crucial for DPML algorithms in \autoref{subsec:magnitude}. Conversely, the results derived from the proposed \ourmethod{} predominantly reside in the top-right region, showcasing an enhancement over the vanilla SAM. This implies that \ourmethod{} can further refine the robustness-generalization trade-off compared to vanilla SAM, which may elucidate its consistent ability to amplify the performance of DPML algorithms with the most substantial improvements. Additional results, obtained when the privacy budget is set to $\epsilon=4$, are provided in Appendix~\ref{app:eps_4}.

\section{Discussion}
\label{sec:discussion}
The current design of \ourmethod{} is specifically targeted towards supervised learning and cannot be directly applied to unsupervised learning techniques like contrastive learning~\cite{Momentum, MoCoV2, MoCoV3}. Thus, the attacker in our scenario needs to have a labeled dataset for pre-training, which might pose a challenge in certain specialized fields. However, the central idea of leveraging parameter robustness to enhance DPML algorithms is a broad concept and not limited to supervised learning. Yet, the question of how to actively induce accumulate parameter robustness in the context of contrastive learning remains open and presents an exciting direction for future research.

Moreover, while our discussion and implementation center on a single-party scenario, it is worth noting that our framework is adaptable to both single-party and multi-party situations. In a multi-party context, the private data originates from multiple sources. Consequently, private fine-tuning processes can be facilitated through federated learning. In these instances, the DP-SGD algorithm would need to be supplanted by the DP-FedAvg algorithm~\cite{Decentralized} to integrate DP noise during aggregation. Although \ourmethod{} provides a promising foundation, the adaptation, execution, and assessment of the methodology in multi-party contexts necessitate further exploration and rigorous validation. Future research endeavors will delve deeper into the complexities and nuances of multi-party scenarios, evaluating the adaptability, scalability, and efficiency of \ourmethod{} in these environments.
\section{Related Work}
\label{sec:related}

\para{Differential Privacy}.
Differential privacy (DP)\cite{DP_survey, DP_Foundations} is a widely-adopted, rigorous mathematical framework that formalizes and measures privacy guarantees, primarily using a parameter termed the privacy budget. It has been employed in various data analysis tasks, such as synthetic dataset generation~\cite{PrivGraph, PrivSyn, PrivTrace, LDPTrace}, marginal release~\cite{CALM}, range query~\cite{AHEAD}, and stream data analysis~\cite{Continuous_Release}. Abadi et al. introduced vanilla DP-SGD~\cite{DPSGD} as the pioneering general DP algorithm for deep learning.
Subsequent research efforts have focused on minimizing the impact of DP on model utility, either by devising new algorithms~\cite{Semi-supervised_DP, more_data, Overbill, Private-kNN} or by relaxing DP definition to suit specific contexts~\cite{Theory_meets_Practice, New_Settings, label_DP}. In contrast, our proposed method, \ourmethod{}, seeks to minimize the impact of DP on model utility by providing a parameter-robust pre-trained model as a service, an approach that has not been explored previously.

\vspace{2pt}
\para{Parameter Robustness}.
Numerous studies have demonstrated that enhancing parameter robustness effectively reduces the generalization gap, thereby improving model generality~\cite{AMP, RMP}. Additional research indicates that improvements in parameter robustness can also mitigate robust overfitting during adversarial training, achieving heightened generality with respect to model robustness against input perturbations~\cite{AWP, RWP, Formalizing_Generalization, Relating_Adversarially}. However, all of these studies utilize the additional benefits conferred by parameter robustness, i.e., generality, without directly leveraging the property of parameter robustness itself. In contrast, the proposed \ourmethod{} is among the first strategies that directly utilize parameter robustness to benefit private learning. This introduces a fresh perspective on the potential applications of the parameter robustness property.

\vspace{2pt}
\para{Pre-trained Model as a Service}.
Adversarial training~\cite{Explaining, PGD} is recognized as a standard method for developing empirically robust classifiers in supervised learning. The core concept involves generating adversarial examples from training instances during the training process and augmenting the training data with these examples. Several studies~\cite{Adversarial_Robustness, Robust_Pre-Training, Kim, REaaS} have generalized adversarial training for the pre-training of robust models in self-supervised learning. Generally, the approach first generates adversarial examples that result in significant loss, which are then used for model pre-training. However, previous studies have focused solely on offering input robustness as a service, neglecting to offer parameter robustness, which is the primary contribution of this paper.

\vspace{2pt}
\para{Attacks and Defenses in Transfer Learning}.
Research has also delved into membership inference attacks within transfer learning scenarios. Liu et al.\cite{EncoderMI} demonstrated a membership inference attack against pre-trained encoders, elucidating the capability to infer whether an input was part of the encoder's training dataset. Concurrently, He et al.\cite{Privacy_CL} introduced a mitigation strategy capable of countering such threats. However, our work pivots the focus towards the privacy risks associated with downstream tasks, rather than those tied to pre-trained encoders. \ourmethod{} represents the first attempt to offer a pre-training method aimed at enhancing defenses against downstream attacks.

\section{Conclusion}
\label{sec:conclusion}
In this study, we unveiled \ourmethod{}, a groundbreaking technique engineered to augment parameter robustness, thereby navigating the traditionally adversarial relationship between Differential Privacy (DP) noise and model utility in Deep Learning. By meticulously reallocating batch sizes for perturbation and gradient calculations, \ourmethod{} refines Sharpness-Aware Minimization (SAM) algorithms, delivering enhanced parameter robustness and, consequently, mitigating the impact of DP noise. Our comprehensive evaluations across several datasets substantiate \ourmethod{}’s capability to substantially bolster the accuracy of DPML algorithms on various downstream tasks, thereby highlighting its potential as a pivotal technique in future privacy-preserving machine learning endeavors.

\bibliographystyle{plain}
\bibliography{bib}

\begin{thebibliography}{10}

\bibitem{DPSGD}
Mart{\'{\i}}n Abadi, Andy Chu, Ian~J. Goodfellow, H.~Brendan McMahan, Ilya Mironov, Kunal Talwar, and Li~Zhang.
\newblock Deep learning with differential privacy.
\newblock In {\em Proceedings of the 2016 {ACM} {SIGSAC} Conference on Computer and Communications Security, Vienna, Austria, October 24-28, 2016}, pages 308--318, 2016.

\bibitem{andrew2021differentially}
Galen Andrew, Om~Thakkar, Brendan McMahan, and Swaroop Ramaswamy.
\newblock Differentially private learning with adaptive clipping.
\newblock {\em Advances in Neural Information Processing Systems}, 34:17455--17466, 2021.

\bibitem{andriushchenko2022towards}
Maksym Andriushchenko and Nicolas Flammarion.
\newblock Towards understanding sharpness-aware minimization.
\newblock In {\em International Conference on Machine Learning}, pages 639--668. PMLR, 2022.

\bibitem{AE_1}
Anish Athalye, Nicholas Carlini, and David~A. Wagner.
\newblock Obfuscated gradients give a false sense of security: Circumventing defenses to adversarial examples.
\newblock In {\em Proceedings of the 35th International Conference on Machine Learning, {ICML} 2018, Stockholmsm{\"{a}}ssan, Stockholm, Sweden, July 10-15, 2018}, pages 274--283, 2018.

\bibitem{Reconstructing}
Borja Balle, Giovanni Cherubin, and Jamie Hayes.
\newblock Reconstructing training data with informed adversaries.
\newblock In {\em 43rd {IEEE} Symposium on Security and Privacy, {SP} 2022, San Francisco, CA, USA, May 22-26, 2022}, pages 1138--1156, 2022.

\bibitem{Evasion}
Battista Biggio, Igino Corona, Davide Maiorca, Blaine Nelson, Nedim Srndic, Pavel Laskov, Giorgio Giacinto, and Fabio Roli.
\newblock Evasion attacks against machine learning at test time.
\newblock In {\em Machine Learning and Knowledge Discovery in Databases - European Conference, {ECML} {PKDD} 2013, Prague, Czech Republic, September 23-27, 2013, Proceedings, Part {III}}, pages 387--402, 2013.

\bibitem{First_Principles}
Nicholas Carlini, Steve Chien, Milad Nasr, Shuang Song, Andreas Terzis, and Florian Tram{\`{e}}r.
\newblock Membership inference attacks from first principles.
\newblock In {\em 43rd {IEEE} Symposium on Security and Privacy, {SP} 2022, San Francisco, CA, USA, May 22-26, 2022}, pages 1897--1914, 2022.

\bibitem{Extracting}
Nicholas Carlini, Florian Tram{\`{e}}r, Eric Wallace, Matthew Jagielski, Ariel Herbert{-}Voss, Katherine Lee, Adam Roberts, Tom~B. Brown, Dawn Song, {\'{U}}lfar Erlingsson, Alina Oprea, and Colin Raffel.
\newblock Extracting training data from large language models.
\newblock In {\em 30th {USENIX} Security Symposium, {USENIX} Security 2021, August 11-13, 2021}, pages 2633--2650, 2021.

\bibitem{iGPT}
Mark Chen, Alec Radford, Rewon Child, Jeffrey Wu, Heewoo Jun, David Luan, and Ilya Sutskever.
\newblock Generative pretraining from pixels.
\newblock In {\em Proceedings of the 37th International Conference on Machine Learning, {ICML} 2020, 13-18 July 2020, Virtual Event}, pages 1691--1703, 2020.

\bibitem{Adversarial_Robustness}
Tianlong Chen, Sijia Liu, Shiyu Chang, Yu~Cheng, Lisa Amini, and Zhangyang Wang.
\newblock Adversarial robustness: From self-supervised pre-training to fine-tuning.
\newblock In {\em 2020 {IEEE/CVF} Conference on Computer Vision and Pattern Recognition, {CVPR} 2020, Seattle, WA, USA, June 13-19, 2020}, pages 696--705, 2020.

\bibitem{MoCoV2}
Xinlei Chen, Haoqi Fan, Ross~B. Girshick, and Kaiming He.
\newblock Improved baselines with momentum contrastive learning.
\newblock {\em CoRR}, abs/2003.04297, 2020.

\bibitem{MoCoV3}
Xinlei Chen, Saining Xie, and Kaiming He.
\newblock An empirical study of training self-supervised vision transformers.
\newblock In {\em 2021 {IEEE/CVF} International Conference on Computer Vision, {ICCV} 2021, Montreal, QC, Canada, October 10-17, 2021}, pages 9620--9629, 2021.

\bibitem{STL-10}
Adam Coates, Andrew~Y. Ng, and Honglak Lee.
\newblock An analysis of single-layer networks in unsupervised feature learning.
\newblock In {\em Proceedings of the Fourteenth International Conference on Artificial Intelligence and Statistics, {AISTATS} 2011, Fort Lauderdale, USA, April 11-13, 2011}, pages 215--223, 2011.

\bibitem{NLP}
Ronan Collobert, Jason Weston, L{\'{e}}on Bottou, Michael Karlen, Koray Kavukcuoglu, and Pavel~P. Kuksa.
\newblock Natural language processing (almost) from scratch.
\newblock {\em J. Mach. Learn. Res.}, 12:2493--2537, 2011.

\bibitem{Unlocking}
Soham De, Leonard Berrada, Jamie Hayes, Samuel~L. Smith, and Borja Balle.
\newblock Unlocking high-accuracy differentially private image classification through scale.
\newblock {\em CoRR}, abs/2204.13650, 2022.

\bibitem{ViT}
Alexey Dosovitskiy, Lucas Beyer, Alexander Kolesnikov, Dirk Weissenborn, Xiaohua Zhai, Thomas Unterthiner, Mostafa Dehghani, Matthias Minderer, Georg Heigold, Sylvain Gelly, Jakob Uszkoreit, and Neil Houlsby.
\newblock An image is worth 16x16 words: Transformers for image recognition at scale.
\newblock In {\em 9th International Conference on Learning Representations, {ICLR} 2021, Virtual Event, Austria, May 3-7, 2021}, 2021.

\bibitem{AHEAD}
Linkang Du, Zhikun Zhang, Shaojie Bai, Changchang Liu, Shouling Ji, Peng Cheng, and Jiming Chen.
\newblock {AHEAD:} adaptive hierarchical decomposition for range query under local differential privacy.
\newblock In {\em {CCS} '21: 2021 {ACM} {SIGSAC} Conference on Computer and Communications Security, Virtual Event, Republic of Korea, November 15 - 19, 2021}, pages 1266--1288, 2021.

\bibitem{LDPTrace}
Yuntao Du, Yujia Hu, Zhikun Zhang, Ziquan Fang, Lu~Chen, Baihua Zheng, and Yunjun Gao.
\newblock Ldptrace: Locally differentially private trajectory synthesis.
\newblock {\em Proc. {VLDB} Endow.}, 16(8):1897--1909, 2023.

\bibitem{DP_survey}
Cynthia Dwork.
\newblock Differential privacy: {A} survey of results.
\newblock In {\em Theory and Applications of Models of Computation, 5th International Conference, {TAMC} 2008, Xi'an, China, April 25-29, 2008. Proceedings}, pages 1--19, 2008.

\bibitem{New_Settings}
Cynthia Dwork.
\newblock Differential privacy in new settings.
\newblock In {\em Proceedings of the Twenty-First Annual {ACM-SIAM} Symposium on Discrete Algorithms, {SODA} 2010, Austin, Texas, USA, January 17-19, 2010}, pages 174--183, 2010.

\bibitem{DP}
Cynthia Dwork, Frank McSherry, Kobbi Nissim, and Adam~D. Smith.
\newblock Calibrating noise to sensitivity in private data analysis.
\newblock In {\em Theory of Cryptography, Third Theory of Cryptography Conference, {TCC} 2006, New York, NY, USA, March 4-7, 2006, Proceedings}, pages 265--284, 2006.

\bibitem{DP_Foundations}
Cynthia Dwork and Aaron Roth.
\newblock The algorithmic foundations of differential privacy.
\newblock {\em Found. Trends Theor. Comput. Sci.}, 9(3-4):211--407, 2014.

\bibitem{sam}
Pierre Foret, Ariel Kleiner, Hossein Mobahi, and Behnam Neyshabur.
\newblock Sharpness-aware minimization for efficiently improving generalization.
\newblock {\em arXiv preprint arXiv:2010.01412}, 2020.

\bibitem{label_DP}
Badih Ghazi, Noah Golowich, Ravi Kumar, Pasin Manurangsi, and Chiyuan Zhang.
\newblock Deep learning with label differential privacy.
\newblock In {\em Advances in Neural Information Processing Systems 34: Annual Conference on Neural Information Processing Systems 2021, NeurIPS 2021, December 6-14, 2021, virtual}, pages 27131--27145, 2021.

\bibitem{Adversarial}
Ian~J. Goodfellow, Jonathon Shlens, and Christian Szegedy.
\newblock Explaining and harnessing adversarial examples.
\newblock In {\em 3rd International Conference on Learning Representations, {ICLR} 2015, San Diego, CA, USA, May 7-9, 2015, Conference Track Proceedings}, 2015.

\bibitem{Explaining}
Ian~J. Goodfellow, Jonathon Shlens, and Christian Szegedy.
\newblock Explaining and harnessing adversarial examples.
\newblock In {\em 3rd International Conference on Learning Representations, {ICLR} 2015, San Diego, CA, USA, May 7-9, 2015, Conference Track Proceedings}, 2015.

\bibitem{lr_bs}
Priya Goyal, Piotr Doll{\'{a}}r, Ross~B. Girshick, Pieter Noordhuis, Lukasz Wesolowski, Aapo Kyrola, Andrew Tulloch, Yangqing Jia, and Kaiming He.
\newblock Accurate, large minibatch {SGD:} training imagenet in 1 hour.
\newblock {\em CoRR}, abs/1706.02677, 2017.

\bibitem{Momentum}
Kaiming He, Haoqi Fan, Yuxin Wu, Saining Xie, and Ross~B. Girshick.
\newblock Momentum contrast for unsupervised visual representation learning.
\newblock In {\em 2020 {IEEE/CVF} Conference on Computer Vision and Pattern Recognition, {CVPR} 2020, Seattle, WA, USA, June 13-19, 2020}, pages 9726--9735, 2020.

\bibitem{ResNet}
Kaiming He, Xiangyu Zhang, Shaoqing Ren, and Jian Sun.
\newblock Deep residual learning for image recognition.
\newblock In {\em 2016 {IEEE} Conference on Computer Vision and Pattern Recognition, {CVPR} 2016, Las Vegas, NV, USA, June 27-30, 2016}, pages 770--778, 2016.

\bibitem{Privacy_CL}
Xinlei He and Yang Zhang.
\newblock Quantifying and mitigating privacy risks of contrastive learning.
\newblock In {\em {CCS} '21: 2021 {ACM} {SIGSAC} Conference on Computer and Communications Security, Virtual Event, Republic of Korea, November 15 - 19, 2021}, pages 845--863, 2021.

\bibitem{Robust_Pre-Training}
Ziyu Jiang, Tianlong Chen, Ting Chen, and Zhangyang Wang.
\newblock Robust pre-training by adversarial contrastive learning.
\newblock In {\em Advances in Neural Information Processing Systems 33: Annual Conference on Neural Information Processing Systems 2020, NeurIPS 2020, December 6-12, 2020, virtual}, 2020.

\bibitem{Kim}
Minseon Kim, Jihoon Tack, and Sung~Ju Hwang.
\newblock Adversarial self-supervised contrastive learning.
\newblock In {\em Advances in Neural Information Processing Systems 33: Annual Conference on Neural Information Processing Systems 2020, NeurIPS 2020, December 6-12, 2020, virtual}, 2020.

\bibitem{krizhevsky2009learning}
Alex Krizhevsky.
\newblock Learning multiple layers of features from tiny images.
\newblock Tech. report, University of Toronto, 2009.

\bibitem{Pretrain_DP}
Alexey Kurakin, Steve Chien, Shuang Song, Roxana Geambasu, Andreas Terzis, and Abhradeep Thakurta.
\newblock Toward training at imagenet scale with differential privacy.
\newblock {\em CoRR}, abs/2201.12328, 2022.

\bibitem{RMP}
Tao Li, Weihao Yan, Zehao Lei, Yingwen Wu, Kun Fang, Ming Yang, and Xiaolin Huang.
\newblock Efficient generalization improvement guided by random weight perturbation.
\newblock {\em CoRR}, abs/2211.11489, 2022.

\bibitem{LLM_DP}
Xuechen Li, Florian Tram{\`{e}}r, Percy Liang, and Tatsunori Hashimoto.
\newblock Large language models can be strong differentially private learners.
\newblock In {\em The Tenth International Conference on Learning Representations, {ICLR} 2022, Virtual Event, April 25-29, 2022}, 2022.

\bibitem{Label-Only}
Zheng Li and Yang Zhang.
\newblock Membership leakage in label-only exposures.
\newblock In {\em {CCS} '21: 2021 {ACM} {SIGSAC} Conference on Computer and Communications Security, Virtual Event, Republic of Korea, November 15 - 19, 2021}, pages 880--895, 2021.

\bibitem{EncoderMI}
Hongbin Liu, Jinyuan Jia, Wenjie Qu, and Neil~Zhenqiang Gong.
\newblock Encodermi: Membership inference against pre-trained encoders in contrastive learning.
\newblock In {\em {CCS} '21: 2021 {ACM} {SIGSAC} Conference on Computer and Communications Security, Virtual Event, Republic of Korea, November 15 - 19, 2021}, pages 2081--2095, 2021.

\bibitem{ML-Doctor}
Yugeng Liu, Rui Wen, Xinlei He, Ahmed Salem, Zhikun Zhang, Michael Backes, Emiliano~De Cristofaro, Mario Fritz, and Yang Zhang.
\newblock Ml-doctor: Holistic risk assessment of inference attacks against machine learning models.
\newblock In {\em 31st {USENIX} Security Symposium, {USENIX} Security 2022, Boston, MA, USA, August 10-12, 2022}, pages 4525--4542, 2022.

\bibitem{Theory_meets_Practice}
Ashwin Machanavajjhala, Daniel Kifer, John~M. Abowd, Johannes Gehrke, and Lars Vilhuber.
\newblock Privacy: Theory meets practice on the map.
\newblock In {\em Proceedings of the 24th International Conference on Data Engineering, {ICDE} 2008, April 7-12, 2008, Canc{\'{u}}n, Mexico}, pages 277--286, 2008.

\bibitem{PGD}
Aleksander Madry, Aleksandar Makelov, Ludwig Schmidt, Dimitris Tsipras, and Adrian Vladu.
\newblock Towards deep learning models resistant to adversarial attacks.
\newblock In {\em 6th International Conference on Learning Representations, {ICLR} 2018, Vancouver, BC, Canada, April 30 - May 3, 2018, Conference Track Proceedings}, 2018.

\bibitem{Decentralized}
Brendan McMahan, Eider Moore, Daniel Ramage, Seth Hampson, and Blaise~Ag{\"{u}}era y~Arcas.
\newblock Communication-efficient learning of deep networks from decentralized data.
\newblock In {\em Proceedings of the 20th International Conference on Artificial Intelligence and Statistics, {AISTATS} 2017, 20-22 April 2017, Fort Lauderdale, FL, {USA}}, pages 1273--1282, 2017.

\bibitem{DP_Transfer}
Harsh Mehta, Abhradeep Thakurta, Alexey Kurakin, and Ashok Cutkosky.
\newblock Large scale transfer learning for differentially private image classification.
\newblock {\em CoRR}, abs/2205.02973, 2022.

\bibitem{Gaussian_Mechanism}
Ilya Mironov, Kunal Talwar, and Li~Zhang.
\newblock R{\'{e}}nyi differential privacy of the sampled gaussian mechanism.
\newblock {\em CoRR}, abs/1908.10530, 2019.

\bibitem{svhn}
Yuval Netzer, Tao Wang, Adam Coates, Alessandro Bissacco, Bo~Wu, and Andrew~Y Ng.
\newblock Reading digits in natural images with unsupervised feature learning.
\newblock In {\em Advances in Neural Information Processing Systems (NIPS)}, pages 2843--2851, 2011.

\bibitem{transfer}
Sinno~Jialin Pan and Qiang Yang.
\newblock A survey on transfer learning.
\newblock {\em {IEEE} Trans. Knowl. Data Eng.}, 22(10):1345--1359, 2010.

\bibitem{Semi-supervised_DP}
Nicolas Papernot, Mart{\'{\i}}n Abadi, {\'{U}}lfar Erlingsson, Ian~J. Goodfellow, and Kunal Talwar.
\newblock Semi-supervised knowledge transfer for deep learning from private training data.
\newblock In {\em 5th International Conference on Learning Representations, {ICLR} 2017, Toulon, France, April 24-26, 2017, Conference Track Proceedings}, 2017.

\bibitem{Limitations_AE}
Nicolas Papernot, Patrick~D. McDaniel, Somesh Jha, Matt Fredrikson, Z.~Berkay Celik, and Ananthram Swami.
\newblock The limitations of deep learning in adversarial settings.
\newblock In {\em {IEEE} European Symposium on Security and Privacy, EuroS{\&}P 2016, Saarbr{\"{u}}cken, Germany, March 21-24, 2016}, pages 372--387, 2016.

\bibitem{qu2023reaas}
Wenjie Qu, Jinyuan Jia, and Neil~Zhenqiang Gong.
\newblock Reaas: Enabling adversarially robust downstream classifiers via robust encoder as a service.
\newblock {\em arXiv preprint arXiv:2301.02905}, 2023.

\bibitem{REaaS}
Wenjie Qu, Jinyuan Jia, and Neil~Zhenqiang Gong.
\newblock Reaas: Enabling adversarially robust downstream classifiers via robust encoder as a service.
\newblock In {\em 30th Annual Network and Distributed System Security Symposium, {NDSS} 2023, San Diego, California, USA, February 27 - March 3, 2023}, 2023.

\bibitem{RCNN}
Shaoqing Ren, Kaiming He, Ross~B. Girshick, and Jian Sun.
\newblock Faster {R-CNN:} towards real-time object detection with region proposal networks.
\newblock In {\em Advances in Neural Information Processing Systems 28: Annual Conference on Neural Information Processing Systems 2015, December 7-12, 2015, Montreal, Quebec, Canada}, pages 91--99, 2015.

\bibitem{Relating_Adversarially}
David Stutz, Matthias Hein, and Bernt Schiele.
\newblock Relating adversarially robust generalization to flat minima.
\newblock In {\em 2021 {IEEE/CVF} International Conference on Computer Vision, {ICCV} 2021, Montreal, QC, Canada, October 10-17, 2021}, pages 7787--7797, 2021.

\bibitem{Intriguing}
Christian Szegedy, Wojciech Zaremba, Ilya Sutskever, Joan Bruna, Dumitru Erhan, Ian~J. Goodfellow, and Rob Fergus.
\newblock Intriguing properties of neural networks.
\newblock In {\em 2nd International Conference on Learning Representations, {ICLR} 2014, Banff, AB, Canada, April 14-16, 2014, Conference Track Proceedings}, 2014.

\bibitem{more_data}
Florian Tram{\`{e}}r and Dan Boneh.
\newblock Differentially private learning needs better features (or much more data).
\newblock In {\em 9th International Conference on Learning Representations, {ICLR} 2021, Virtual Event, Austria, May 3-7, 2021}, 2021.

\bibitem{Formalizing_Generalization}
Yu{-}Lin Tsai, Chia{-}Yi Hsu, Chia{-}Mu Yu, and Pin{-}Yu Chen.
\newblock Formalizing generalization and adversarial robustness of neural networks to weight perturbations.
\newblock In {\em Advances in Neural Information Processing Systems 34: Annual Conference on Neural Information Processing Systems 2021, NeurIPS 2021, December 6-14, 2021, virtual}, pages 19692--19704, 2021.

\bibitem{PrivTrace}
Haiming Wang, Zhikun Zhang, Tianhao Wang, Shibo He, Michael Backes, Jiming Chen, and Yang Zhang.
\newblock Privtrace: Differentially private trajectory synthesis by adaptive markov models.
\newblock In {\em 32nd {USENIX} Security Symposium, {USENIX} Security 2023, Anaheim, CA, USA, August 9-11, 2023}, 2023.

\bibitem{Continuous_Release}
Tianhao Wang, Joann~Qiongna Chen, Zhikun Zhang, Dong Su, Yueqiang Cheng, Zhou Li, Ninghui Li, and Somesh Jha.
\newblock Continuous release of data streams under both centralized and local differential privacy.
\newblock In {\em {CCS} '21: 2021 {ACM} {SIGSAC} Conference on Computer and Communications Security, Virtual Event, Republic of Korea, November 15 - 19, 2021}, pages 1237--1253, 2021.

\bibitem{DPMLBench}
Chengkun Wei, Minghu Zhao, Zhikun Zhang, Min Chen, Wenlong Meng, Bo~Liu, Yuan Fan, and Wenzhi Chen.
\newblock Dpmlbench: Holistic evaluation of differentially private machine learning.
\newblock {\em CoRR}, abs/2305.05900, 2023.

\bibitem{AWP}
Dongxian Wu, Shu{-}Tao Xia, and Yisen Wang.
\newblock Adversarial weight perturbation helps robust generalization.
\newblock In {\em Advances in Neural Information Processing Systems 33: Annual Conference on Neural Information Processing Systems 2020, NeurIPS 2020, December 6-12, 2020, virtual}, 2020.

\bibitem{transfer_theory}
Xuetong Wu, Jonathan~H Manton, Uwe Aickelin, and Jingge Zhu.
\newblock An information-theoretic analysis for transfer learning: Error bounds and applications.
\newblock {\em arXiv preprint arXiv:2207.05377}, 2022.

\bibitem{RWP}
Chaojian Yu, Bo~Han, Mingming Gong, Li~Shen, Shiming Ge, Du~Bo, and Tongliang Liu.
\newblock Robust weight perturbation for adversarial training.
\newblock In {\em Proceedings of the Thirty-First International Joint Conference on Artificial Intelligence, {IJCAI} 2022, Vienna, Austria, 23-29 July 2022}, pages 3688--3694, 2022.

\bibitem{DP_finetune}
Da~Yu, Saurabh Naik, Arturs Backurs, Sivakanth Gopi, Huseyin~A. Inan, Gautam Kamath, Janardhan Kulkarni, Yin~Tat Lee, Andre Manoel, Lukas Wutschitz, Sergey Yekhanin, and Huishuai Zhang.
\newblock Differentially private fine-tuning of language models.
\newblock In {\em The Tenth International Conference on Learning Representations, {ICLR} 2022, Virtual Event, April 25-29, 2022}, 2022.

\bibitem{Overbill}
Da~Yu, Huishuai Zhang, Wei Chen, and Tie{-}Yan Liu.
\newblock Do not let privacy overbill utility: Gradient embedding perturbation for private learning.
\newblock In {\em 9th International Conference on Learning Representations, {ICLR} 2021, Virtual Event, Austria, May 3-7, 2021}, 2021.

\bibitem{yu2021not}
Da~Yu, Huishuai Zhang, Wei Chen, and Tie-Yan Liu.
\newblock Do not let privacy overbill utility: Gradient embedding perturbation for private learning.
\newblock In {\em International Conference on Learning Representations (ICLR)}, 2021.

\bibitem{yu2019differentially}
Lei Yu, Ling Liu, Calton Pu, Mehmet~Emre Gursoy, and Stacey Truex.
\newblock Differentially private model publishing for deep learning.
\newblock In {\em 2019 IEEE Symposium on Security and Privacy (SP)}, pages 332--349. IEEE, 2019.

\bibitem{PrivGraph}
Quan Yuan, Zhikun Zhang, Linkang Du, Min Chen, Peng Cheng, and Mingyang Sun.
\newblock Privgraph: Differentially private graph data publication by exploiting community information.
\newblock In {\em 32nd {USENIX} Security Symposium, {USENIX} Security 2023, Anaheim, CA, USA, August 9-11, 2023}, 2023.

\bibitem{zhang2017efficient}
Jiaqi Zhang, Kai Zheng, Wenlong Mou, and Liwei Wang.
\newblock Efficient private erm for smooth objectives.
\newblock In {\em Proceedings of the 26th International Joint Conference on Artificial Intelligence}, pages 3922--3928, 2017.

\bibitem{CALM}
Zhikun Zhang, Tianhao Wang, Ninghui Li, Shibo He, and Jiming Chen.
\newblock {CALM:} consistent adaptive local marginal for marginal release under local differential privacy.
\newblock In {\em Proceedings of the 2018 {ACM} {SIGSAC} Conference on Computer and Communications Security, {CCS} 2018, Toronto, ON, Canada, October 15-19, 2018}, pages 212--229, 2018.

\bibitem{PrivSyn}
Zhikun Zhang, Tianhao Wang, Ninghui Li, Jean Honorio, Michael Backes, Shibo He, Jiming Chen, and Yang Zhang.
\newblock Privsyn: Differentially private data synthesis.
\newblock In {\em 30th {USENIX} Security Symposium, {USENIX} Security 2021, August 11-13, 2021}, pages 929--946, 2021.

\bibitem{AMP}
Yaowei Zheng, Richong Zhang, and Yongyi Mao.
\newblock Regularizing neural networks via adversarial model perturbation.
\newblock In {\em {IEEE} Conference on Computer Vision and Pattern Recognition, {CVPR} 2021, virtual, June 19-25, 2021}, pages 8156--8165, 2021.

\bibitem{Private-kNN}
Yuqing Zhu, Xiang Yu, Manmohan Chandraker, and Yu{-}Xiang Wang.
\newblock Private-knn: Practical differential privacy for computer vision.
\newblock In {\em 2020 {IEEE/CVF} Conference on Computer Vision and Pattern Recognition, {CVPR} 2020, Seattle, WA, USA, June 13-19, 2020}, pages 11851--11859, 2020.

\end{thebibliography}

\appendix
\newpage
\onecolumn
\section*{Appendix}
\label{sec:appendix}

\section{Theoretical Results}

\subsection{Proof of \autoref{thm:1}}
\label{app:thm1}
We first restate Theorem \autoref{thm:1} as follows. 
\begin{theorem}[Formal version of \autoref{thm:1}]
\label{thm:11}
Denote $\hat\beta:=(\rho^2\beta_2 + \beta\beta_1)$. Select step sizes $\eta_1 = \eta_2 = 1/(4\hat\beta)$. Then Algorithm \ref{alg:DPAdapter} enjoys the following utility bound:
\begin{align}
        &\frac{1}{T}\sum_{t=1}^T \mathbb{E}(L_\cD(\btheta_t) - \inf_{\hat \btheta} L_\cD(\hat \btheta) )\leq \mu\cdot\bigg(\frac{L_\cD(\btheta_0)}{T} + \frac{\hat\sigma^2}{16\hat\beta}(1/|\cB_2| + 1/|\cB_1|)\bigg).\notag
    \end{align}
\end{theorem}
To prove Theorem \ref{thm:11}, we need the following lemma which estimates the Lipschitz constant and smoothness constant of $L_i$. 
\begin{lemma}\label{lemma:1}
    For each $i$, $L_i$ is $\beta_1\rho$-Lipschitz continuous and $(\rho^2\beta_2+\beta\beta_1)$-smooth. 
\end{lemma}
\begin{proof}
     We have for all $\btheta, \btheta'$, 
    \begin{align}
        |L_i(\btheta) - L_i(\btheta')| &= |\ell(f_{\btheta}(x_i), y_i) - \ell(f_{\btheta'}(x_i), y_i)|\notag \\
        &\leq \beta_1|f_{\btheta}(x_i) - f_{\btheta'}(x_i)| \notag \\
        & \leq \beta_1\rho\|\btheta - \btheta'\|_2,\notag
    \end{align}
    and
    \begin{align}
        &\|\nabla L_i(\btheta) - \nabla L_i(\btheta')\|_2\notag \\
        & = \|\nabla_x \ell(f_{\btheta}(x_i), y_i) \nabla f_{\btheta}(x_i) - \nabla_x \ell(f_{\btheta'}(x_i), y_i) \nabla f_{\btheta'}(x_i)\|_2\notag \\
        & \leq \|\nabla_x \ell(f_{\btheta}(x_i), y_i) \nabla f_{\btheta}(x_i) - \nabla_x \ell(f_{\btheta'}(x_i), y_i) \nabla f_{\btheta}(x_i)\|_2 + \|\nabla_x \ell(f_{\btheta'}(x_i), y_i) \nabla f_{\btheta}(x_i) - \nabla_x \ell(f_{\btheta'}(x_i), y_i) \nabla f_{\btheta'}(x_i)\|_2\notag \\
        & \leq \|\nabla f_{\btheta}(x_i)\|_2|\nabla_x \ell(f_{\btheta}(x_i), y_i) - \nabla_x \ell(f_{\btheta'}(x_i), y_i)| + |\nabla_x \ell(f_{\btheta'}(x_i), y_i)|\|\nabla f_{\btheta}(x_i) - \nabla f_{\btheta'}(x_i)\|_2\notag \\
        & \leq \rho^2\beta_2\|\btheta - \btheta'\|_2 + \beta_1\beta\|\btheta - \btheta'\|_2.\notag
    \end{align}
    Therefore, $L_i$ is $\beta_1\rho$-Lipschitz continuous and $(\rho^2\beta_2 + \beta\beta_1)$-smooth.
\end{proof}
We now begin to prove Theorem \ref{thm:11}. 
\begin{proof}[Proof of Theorem \ref{thm:11}]
The proof is adapted from \cite{andriushchenko2022towards}. First, by Lemma \ref{lemma:1} we know that $L_i$ is $\beta_1\rho$-Lipschitz continuous and $(\rho^2\beta_2 + \beta\beta_1)$-smooth. By Lemma 13 in \cite{andriushchenko2022towards}, we have
    \begin{align}
        \mathbb{E}\langle \nabla L_{\cB_2}(\btheta_t+\eta_2\cdot \nabla L_{\cB_1}(\btheta_t)), \nabla L_{\cD}(\btheta_t)\rangle \geq (1-\hat\beta\eta_2)\|L_{\cD}(\btheta_t)\|_2^2 - \frac{\hat\beta^2\eta_2^2\hat\sigma^2}{2|\cB_1|}.\label{help:1}
    \end{align}
    By Lemma 14 in \cite{andriushchenko2022towards}, denote $\btheta_{t+1/2}:=\btheta_t+\eta_2\cdot \nabla L_{\cB_1}(\btheta_t)$, then we have for all $\eta_1 \leq 1/(2\hat\beta)$ and $\eta_2 \leq 1/(2\hat\beta)$, 
    \begin{align}
        \mathbb{E}L_{\cD}(\btheta_{t+1})&\leq \mathbb{E}L_{\cD}(\btheta_t)+\eta_1^2\hat\beta\hat\sigma^2/|\cB_2| -\eta_1^2\hat\beta\mathbb{E}\|\nabla L_\cD(\btheta_t)\|_2^2 -\eta_1(1-2\eta_1\hat\beta)\mathbb{E}\langle\nabla L_\cD(\btheta_{t+1/2}), \nabla L_\cD(\btheta_{t})\rangle\notag \\
        &\quad +\eta_1^2\hat\beta\mathbb{E}\|\nabla L_\cD(\btheta_{t+1/2}) - \nabla L_\cD(\btheta_t)\|_2^2\notag\\
&\leq \mathbb{E}L_{\cD}(\btheta_t)+\eta_1^2\hat\beta\hat\sigma^2/|\cB_2| -\eta_1^2\hat\beta\mathbb{E}\|\nabla L_\cD(\btheta_t)\|_2^2 -\eta_1(1-2\eta_1\hat\beta)\mathbb{E}\langle\nabla L_\cD(\btheta_{t+1/2}), \nabla L_\cD(\btheta_{t})\rangle\notag \\
        &\quad + \eta_1^2\hat\beta\mathbb{E}\|\btheta_{t+1/2} - \btheta_t\|_2^2\notag\\
        & = \mathbb{E}L_{\cD}(\btheta_t)+\eta_1^2\hat\beta\hat\sigma^2/|\cB_2| -\eta_1^2\hat\beta\mathbb{E}\|\nabla L_\cD(\btheta_t)\|_2^2 -\eta_1(1-2\eta_1\hat\beta)\mathbb{E}\langle\nabla L_\cD(\btheta_{t+1/2}), \nabla L_\cD(\btheta_{t})\rangle\notag \\
        &\quad + \eta_1^2\hat\beta^3\eta_2^2\mathbb{E}\|\nabla L_{\cB_1}(\btheta_t)\|_2^2\notag\\
        &\leq \mathbb{E}L_{\cD}(\btheta_{t})+\eta_1^2\hat\beta\hat\sigma^2/|\cB_2| -\eta_1^2\hat\beta(1-2\hat\beta^2\eta_2^2)\mathbb{E}\|\nabla L_\cD(\btheta_t)\|_2^2 -\eta_1(1-2\eta_1\hat\beta)\mathbb{E}\langle\nabla L_\cD(\btheta_{t+1/2}), \nabla L_\cD(\btheta_{t})\rangle\notag \\
        &\quad  + 2\eta_1^2\hat\beta^3\eta_2^2\hat\sigma^2/|\cB_1|.\notag
    \end{align}
    Then substituting \autoref{help:1} and setting $\eta_1 = \eta_2 = 1/(4\hat\beta)$, we have
    \begin{align}
        &\mathbb{E}L_{\cD}(\btheta_{t+1}) - \mathbb{E}L_{\cD}(\btheta_{t}) \notag \\
        &\leq \hat\sigma^2/(16|\cB_2|\hat\beta) - 1/(16\hat\beta)\cdot 7/8\cdot \mathbb{E}\|\nabla L_\cD(\btheta_t)\|_2^2 - 1/(8\hat\beta)\cdot (3/4\cdot \mathbb{E}\|\nabla L_\cD(\btheta_t)\|_2^2 - \hat\sigma^2/(32|\cB_1|)) + \hat\sigma^2/(128\hat\beta|\cB_1|)\notag \\
        & \leq \hat\sigma^2/(64\hat\beta|\cB_1|) + \hat\sigma^2/(16\hat\beta|\cB_2|) - 1/(8\hat\beta)\cdot \mathbb{E}\|\nabla L_\cD(\btheta_t)\|_2^2.\notag
    \end{align}
    Therefore, we have
    \begin{align}
        \frac{1}{T}\sum_{t=1}^T \mathbb{E}\|\nabla L_\cD(\btheta_t)\|_2^2 \leq \frac{L_\cD(\btheta_0)}{T} + \frac{\hat\sigma^2}{16\hat\beta}(1/|\cB_2| + 1/|\cB_1|).\label{eq:fff}
    \end{align}
    Finally, by using the PL-condition we have $\|\nabla L_\cD(\btheta_t)\|_2^2 \geq 1/\mu\cdot(L_\cD(\btheta_t) - \inf_{\btheta} L_\cD(\btheta))$. Substituting it into \eqref{eq:fff} concludes the proof. 
\end{proof}

\subsection{Proof of \autoref{thm:2}}
\label{app:thm2}

\begin{algorithm}[!tbp]
\small
\caption{Random round DP-SGD}
\label{alg:modisgd}
\begin{algorithmic}[1]
\REQUIRE Training set $\mathcal{D}=\{(\boldsymbol{x},\boldsymbol{y})\}$, loss function $L_i:=\ell(f_{\btheta}(\boldsymbol{x}_i), \boldsymbol{y}_i)$, privacy parameter $(\epsilon, \delta)$, learning rate $\eta$. 
\STATE Randomly draw $R\sim \mathbb{P}$, where
\begin{align}
    \mathbb{P}(R = k+1):=\frac{1}{|\cD|^2}, k = 0,1,\dots, |\cD|^2-1.\notag
\end{align}
\FOR{$t = 0,\dots, R-1$}
\STATE Uniformly randomly draw $i$ from $1,\dots, |\cD|$
\STATE Draw $z_t\in \mathbb{R}^d$ from the following Gaussian distribution:
\begin{align}
    z^t\sim \mathcal{N}(0, 4\beta_1^2\rho^2\log(3|\cD|/\delta)\log(2/\delta)\cdot \mathbf{I})
\end{align}
\STATE Update $\btheta_{t+1}:=\btheta_t - \eta\cdot (\nabla L_i(\btheta_t) + z_t)$
\ENDFOR
\ENSURE $\btheta_{\text{out}} = \btheta_R$
\end{algorithmic}
\end{algorithm}

We present the random round DP-SGD proposed in \cite{zhang2017efficient} as Algorithm \ref{alg:modisgd}. Compared with the standard DP-SGD \cite{DPSGD}, the main difference is the random selected round number $R$ here. The use of a random round number is only due to the need of theoretical proof. Next, we state the formal version of Theorem \autoref{thm:2} as follows. 
\begin{theorem}[Formal version of \autoref{thm:2}]
\label{thm:22}
    Select the step size $\eta = \min\{1/(\rho^2\beta_2 + \beta\beta_1), D_f/(\sigma|\cD|)\}$, where
    \begin{align}
        D_f:=\sqrt{2(L_\cD(\btheta_0) - \min_{\btheta} L_\cD(\btheta))/(\rho^2\beta_2 + \beta\beta_1)},\notag\\
        \sigma:=2\beta_1\rho \sqrt{1+\frac{d\log(3|\cD|/\delta)\log(2\delta)}{\epsilon^2} },\notag
    \end{align}
    then we have the following utility bound:
        \begin{align}
        \mathbb{E}(L_\cD(\btheta_{\text{out}}) - \inf_{\btheta} L_\cD(\btheta) )\leq C\cdot \frac{\rho\sqrt{\rho^2\beta_2 + \beta\beta_1}}{|\cD|\epsilon},\notag
    \end{align}
    where $C = c\cdot\mu\beta_1\sqrt{d\log(|\cD|/\delta)\log(1/\delta)L_\cD(\btheta_0)}$, $c$ is some positive constant, $d$ is the dimension of parameter $\btheta$. 
\end{theorem}
\begin{proof}[Proof of \autoref{thm:22}]
First, by Lemma \ref{lemma:1}, we have $L_i$ is $\beta_1\rho$-Lipschitz continuous and $(\rho^2\beta_2 + \beta\beta_1)$-smooth. Then according to Theorem 5 in \cite{zhang2017efficient}, we have the following utility bound:
    \begin{align}
        \mathbb{E}\|\nabla L_\cD(\btheta_{\text{out}})\|_2^2 \leq c\cdot\bigg(\frac{\beta_1\rho\sqrt{\rho^2\beta_2 + \beta\beta_1}\sqrt{d\log(|\cD|/\delta)\log(1/\delta)L_\cD(\btheta_0)}}{|\cD|\epsilon}\bigg),\label{eq:ggg}
    \end{align}
    where $c$ is some positive constant. 
    Meanwhile, random round DP-SGD is $(\epsilon, \delta)$-DP due to Theorem 4 in \cite{zhang2017efficient}. 
    Finally, by using the PL-condition we have $\|\nabla L_\cD(\btheta_{\text{out}})\|_2^2 \geq 1/\mu\cdot(L_\cD(\btheta_{\text{out}}) - \inf_{\btheta} L_\cD(\btheta))$. Substituting it into \autoref{eq:ggg} concludes the proof. 
\end{proof}

\section{Comparison with Vanilla SAM}

In this experiment, we consider four perturbation magnitudes: 0.1, 0.5, 1.0, and 2.0. The DPML algorithm is fixed at vanilla DP-SGD and the privacy budget is fixed at $\epsilon=4$. The results of using the vanilla SAM are denoted by blue triangles, while the results of the proposed \ourmethod{} are denoted by red diamonds. The result of using a normally pre-trained model is denoted as a purple rectangle.

\label{app:eps_4}
\begin{figure*}[!ht]
\centerline{\includegraphics[width=\linewidth]{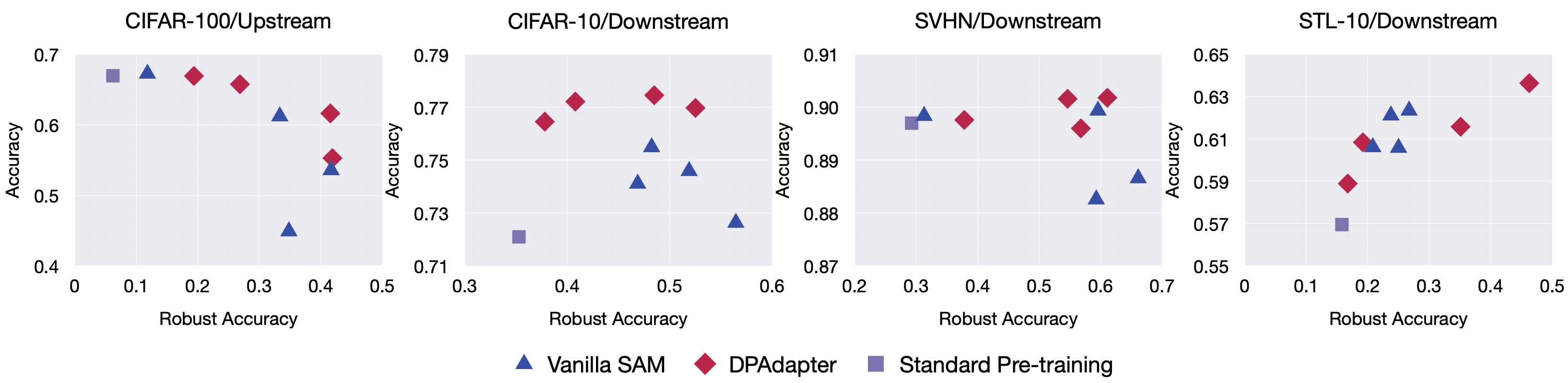}}
\caption{The results of \ourmethod{} and vanilla SAM under different perturbation magnitudes. The results of the vanilla SAM are represented by blue triangles, while the results of the proposed \ourmethod{} are depicted by red diamonds. Results obtained using a vanilla pre-trained model are illustrated by a purple rectangle. All the downstream models have the same privacy budget $\epsilon=4$.}
\label{fig:eps_4}
\end{figure*}

\end{document}